%% file: iclr2025_conference.tex
\documentclass{article} 
\usepackage{iclr2025_conference,times}
\iclrfinalcopy{\iclrfinaltrue}
\input{math_commands.tex}

\usepackage{subcaption}
\usepackage{graphicx}
\usepackage[utf8]{inputenc} 
\usepackage[T1]{fontenc}    
\usepackage{hyperref}       
\usepackage{url}            
\usepackage{booktabs}       
\usepackage{amsfonts}       
\usepackage{nicefrac}       
\usepackage{microtype}      
\usepackage{xcolor}         
\usepackage{pgf,tikz}
\usetikzlibrary{babel}
\usepackage{circuitikz}
\usepackage{siunitx}
\usepackage{amsmath}
\usepackage{amssymb}
\usepackage{mathtools}
\usepackage{amsthm}
\usepackage{bbding}
\usepackage{enumitem}
\usepackage{algorithm}
\usepackage{algorithmic}
\usepackage{wrapfig}
\usepackage{multicol}
\usepackage[normalem]{ulem}
\usepackage[capitalize,noabbrev]{cleveref}

\theoremstyle{plain}
\newtheorem{theorem}{Theorem}[section]

\newtheorem{lemma}[theorem]{Lemma}
\newtheorem{corollary}[theorem]{Corollary}
\theoremstyle{definition}
\newtheorem{definition}[theorem]{Definition}
\newtheorem{assumption}[theorem]{Assumption}
\theoremstyle{remark}
\newtheorem{remark}[theorem]{Remark}

\graphicspath{ {figures/} }
\usepackage[textsize=tiny]{todonotes}

\title{Diminishing Exploration: A Minimalist Approach to Piecewise Stationary Multi-Armed Bandits}

\author{Kuan-Ta Li\textsuperscript{1}, Ping-Chun Hsieh\textsuperscript{2}, Yu-Chih Huang\textsuperscript{1} \\
\textsuperscript{1}Institute of Communications Engineering\\
\textsuperscript{2}Department of Computer Science\\
National Yang Ming Chiao Tung University, Hsinchu, Taiwan\\
\texttt{\{oliver58972.ee12,pinghsieh,jerryhuang\}@nycu.edu.tw} \\
}

\begin{document}

\maketitle

\input{sections/0-abstract}
\input{sections/1-intro}
\input{sections/2-problem}

\input{sections/3-algorithm}
\input{sections/4-analysis}

\input{sections/5-discuss}
\input{sections/5-extension}
\input{sections/6-simulate}

\input{sections/7-conclude}
\input{iclr2025_conference.bbl}
\appendix

\input{appendix/1-changealg}
\input{appendix/extendalg}
\input{appendix/2-proof}
\input{appendix/3-parameter}
\input{appendix/related-work}
\input{appendix/simulation}


\end{document}

%% file: math_commands.tex
\usepackage{amsmath,amsfonts,bm}

\def\eqref#1{equation~\ref{#1}}
\def\Eqref#1{Equation~\ref{#1}}

\def\Algref#1{Algorithm~\ref{#1}}

\def\1{\bm{1}}

\DeclareMathAlphabet{\mathsfit}{\encodingdefault}{\sfdefault}{m}{sl}
\SetMathAlphabet{\mathsfit}{bold}{\encodingdefault}{\sfdefault}{bx}{n}

\def\gK{{\mathcal{K}}}

\def\gO{{\mathcal{O}}}

\def\gR{{\mathcal{R}}}

\def\gT{{\mathcal{T}}}

\newcommand{\E}{\mathbb{E}}

\DeclareMathOperator*{\argmax}{arg\,max}

%% file: sections/0-abstract.tex
\begin{abstract}
    The piecewise-stationary bandit problem is an important variant of the multi-armed bandit problem that further considers abrupt changes in the reward distributions. The main theme of the problem is the trade-off between exploration for detecting environment changes and exploitation of traditional bandit algorithms. While this problem has been extensively investigated, existing works either assume knowledge about the number of change points $M$ or require extremely high computational complexity. In this work, we revisit the piecewise-stationary bandit problem from a minimalist perspective. We propose a novel and generic exploration mechanism, called diminishing exploration, which eliminates the need for knowledge about $M$ and can be used in conjunction with an existing change detection-based algorithm to achieve near-optimal regret scaling. Simulation results show that despite oblivious of $M$, equipping existing algorithms with the proposed diminishing exploration generally achieves better empirical regret than the traditional uniform exploration.
\end{abstract}

%% file: sections/1-intro.tex
\section{Introduction}
\label{sec:intro}
The multi-armed bandit (MAB) problem, a classic formulation of online decision making, involves a decision-maker facing a set of arms with unknown reward distributions, and the challenge is to determine a learning strategy that maximizes the cumulative reward. 
MAB encapsulates the fundamental trade-off between exploiting the current known best arm for immediate reward and exploring other arms for potentially discovering better ones. 
Given the prevalence of such exploration-exploitation dilemma in practice, MAB has served as the abstraction of various real-world sequential decision making problems, such as recommender systems \citep{lu2010contextual}, communication networks \citep{gupta2019link,hashemi2018efficient}, and clinical trials \citep{aziz2021multi}.
In the MAB literature, there are two popular frameworks, namely the stochastic bandits and the adversarial bandits \citep{lattimore2020bandit}: (i) In a standard stochastic bandit model \citep{lai1985asymptotically,auer2002finite}, the rewards of each arm are drawn independently from its underlying reward distribution, which is assumed stationary (i.e., remains fixed throughout the learning process). (ii) In an adversarial model, the reward distribution of each arm is determined by an adversary and could change abruptly after each time step.
To extend and unify the above two frameworks, the \textit{piecewise-stationary bandit} \citep{yu2009piecewise} incorporates \textit{change points} into the bandit model, where the reward distribution of each arm could vary abruptly and arbitrarily at each (unknown) change point and remain fixed between two successive change points.
Accordingly, the piecewise-stationary bandit framework serves as a more realistic setting for a broad class of applications which are neither fully stationary nor fully adversarial, such as recommender systems \citep{xu2020contextual} and dynamic pricing systems \citep{yu2009piecewise}.



Piecewise-stationary bandit has been extensively studied from various perspectives, including the passive methods, e.g., forgetting via discounting \citep{kocsis2006discounted} or a sliding window \citep{garivier2011upper}, and the active methods that leverage change-point detectors, e.g., \citep{liu2018change,cao2019nearly}. 
A more comprehensive survey is deferred to Section \ref{sec:intro:related}.
Despite the rich literature, the existing approaches suffer from the following issues: (i) \textit{Required knowledge of the number of change points}: To adapt exploration to the change frequency, most of the existing works require some tuning based on the knowledge about the number of change points (denoted by $M$ throughout this paper), which could be rather difficult to obtain or estimate in practice.
{(ii) \textit{High computational or algorithmic complexity}: 
AdSwitch \citep{auer2019adaptively} and its enhanced versions \citep{suk2022tracking,abbasi2023new} have been proposed to achieve nearly optimal regret without knowing the number of changes by adopting an elimination approach. However, these approaches are computationally costly as they either relies on a large number of calls of detection mechanism \citep{auer2019adaptively}. On the other hand, MASTER \citep{pmlr-v134-wei21b} serves as a generic black-box approach utilizes multiple hierarchical instances to tackle the issue of unknown number of changes and achieve optimal regret guarantees. This additional algorithmic complexity can lead to high overhead and incur high regret, especially for a small number of segments, and this result is presented in our experiments. These motivate the need for a minimalist design for piecewise-stationary MAB.

{\color{blue}
    \begin{table*}[t]
    \caption{A summary of the regret bounds of various algorithms (R.B.: Regret Bound, S.K.: Segment Knowledge, (DE): Diminishing exploration version, (DE)$^{+}$: Diminishing exploration extension version, P: Passive method, A: Active method, E: Elimination approach, M: Multiple instances). The notation $S$ is the total number of times the optimal arm switches to another.} 
    \label{tab:1}
    \begin{center}
    \begin{small}
    \begin{sc}
    \resizebox{\linewidth}{!}{
    \begin{tabular}{l|lcllc}
    \toprule
    Alg  & type & S.K. free & Complexity & R.B. $\tilde{\mathcal{O}}\left(\cdot\right)$ &Reference \\
    \midrule
    D-UCB               & P    &\colorbox{white}{\XSolidBrush}     &\colorbox{white}{$\mathcal{O}\left(KT\right)$} & \colorbox{white}{$\sqrt{MT}$}     & \citep{kocsis2006discounted}\\
    SW-UCB              & P    &\colorbox{white}{\XSolidBrush}     &\colorbox{white}{$\mathcal{O}\left(KT\right)$} & \colorbox{white}{$\sqrt{MT}$}     & \citep{garivier2011upper}\\
    D-TS                & P    &\colorbox{white}{\XSolidBrush}     &\colorbox{white}{$\mathcal{O}\left(KT\right)$} & \colorbox{white}{$\sqrt{MT}$}     & \citep{qi2023discounted}\\
    \midrule
    AdSwitch            & E    & \colorbox{white}{\Checkmark}      &\colorbox{white}{$\mathcal{O}\left(KT^{4}\right)$} & \colorbox{white}{$\sqrt{MT}$} & \citep{auer2019adaptively}\\
    META         & E    & \colorbox{white}{\Checkmark}      &\colorbox{white}{$\mathcal{O}\left(KT^2\right)$}                             & \colorbox{white}{$\sqrt{ST}$} & \citep{suk2023tracking}\\
    ArmSwitch     & E    & \colorbox{white}{\Checkmark}      &\colorbox{white}{$\mathcal{O}\left(K^2T^2\right)$}            & \colorbox{white}{$\sqrt{ST}$} & \citep{abbasi2023new}\\    
    \midrule
    Master        &M & \colorbox{white}{\Checkmark}      &\colorbox{white}{$\mathcal{O}\left(KT\right)$} &\colorbox{white}{$\min\left\{\sqrt{MT},\Delta^{1/3}T^{2/3}+\sqrt{T}\right\}$} &\citep{pmlr-v134-wei21b}\\
    \midrule
    CUSUM-UCB           & A    & \colorbox{white}{\XSolidBrush}   &\colorbox{white}{$\mathcal{O}\left(KT^{2}\right)$}   & \colorbox{white}{$\sqrt{MT}$} & \citep{liu2018change}  \\
    M-UCB               & A    & \colorbox{white}{\XSolidBrush}   &\colorbox{white}{$\mathcal{O}\left(KT\right)$}       & \colorbox{white}{$\sqrt{MT}$} & \citep{cao2019nearly} \\
    GLR-klUCB           & A    & \colorbox{white}{\Checkmark}     &\colorbox{white}{$\mathcal{O}\left(KT^{2}\right)$}   & \colorbox{white}{$\sqrt{MT}$} &\citep{besson2022efficient}  \\
    \textbf{M-UCB {(DE)}}        & A    & \colorbox{yellow}{\Checkmark}     &\colorbox{yellow}{$\mathcal{O}\left(KT\right)$}       & \colorbox{yellow}{$\sqrt{MT}$} &Ours\\
    GLR-UCB {(DE)}      & A    & \colorbox{white}{\Checkmark}     &\colorbox{white}{$\mathcal{O}\left(KT^{2}\right)$}   & \colorbox{white}{$\sqrt{MT}$} &Ours\\
    \textbf{M-UCB {(DE)$^{+}$}}  & A    & \colorbox{yellow}{\Checkmark}     &\colorbox{yellow}{$\mathcal{O}\left(KT\right)$}       & \colorbox{yellow}{$\sqrt{ST}$} &Ours\\
    GLR-UCB {(DE)$^{+}$}& A    & \colorbox{white}{\Checkmark}     &\colorbox{white}{$\mathcal{O}\left(KT^{2}\right)$}   & \colorbox{white}{$\sqrt{ST}$} &Ours\\

    \bottomrule
    \end{tabular}
    }
    \end{sc}
    \end{small}
    \end{center}
    \vskip -0.1in
    \end{table*}
}

\begin{wrapfigure}[15]{r}{0.33\textwidth}
\hspace{-0.5cm}
\begin{subfigure}{0.5\textwidth}
    \vspace{-2cm}
    \includegraphics[width=0.8\textwidth]{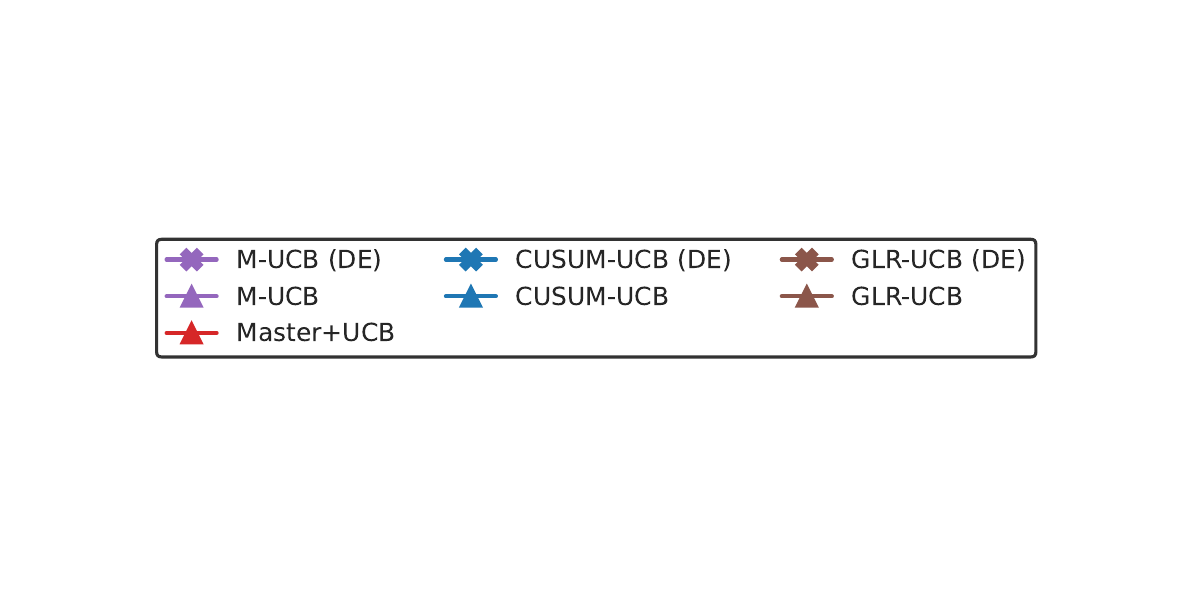}
\end{subfigure}
\begin{subfigure}{0.33\textwidth}
    \vspace{-1.5cm}
    \includegraphics[width=\textwidth]{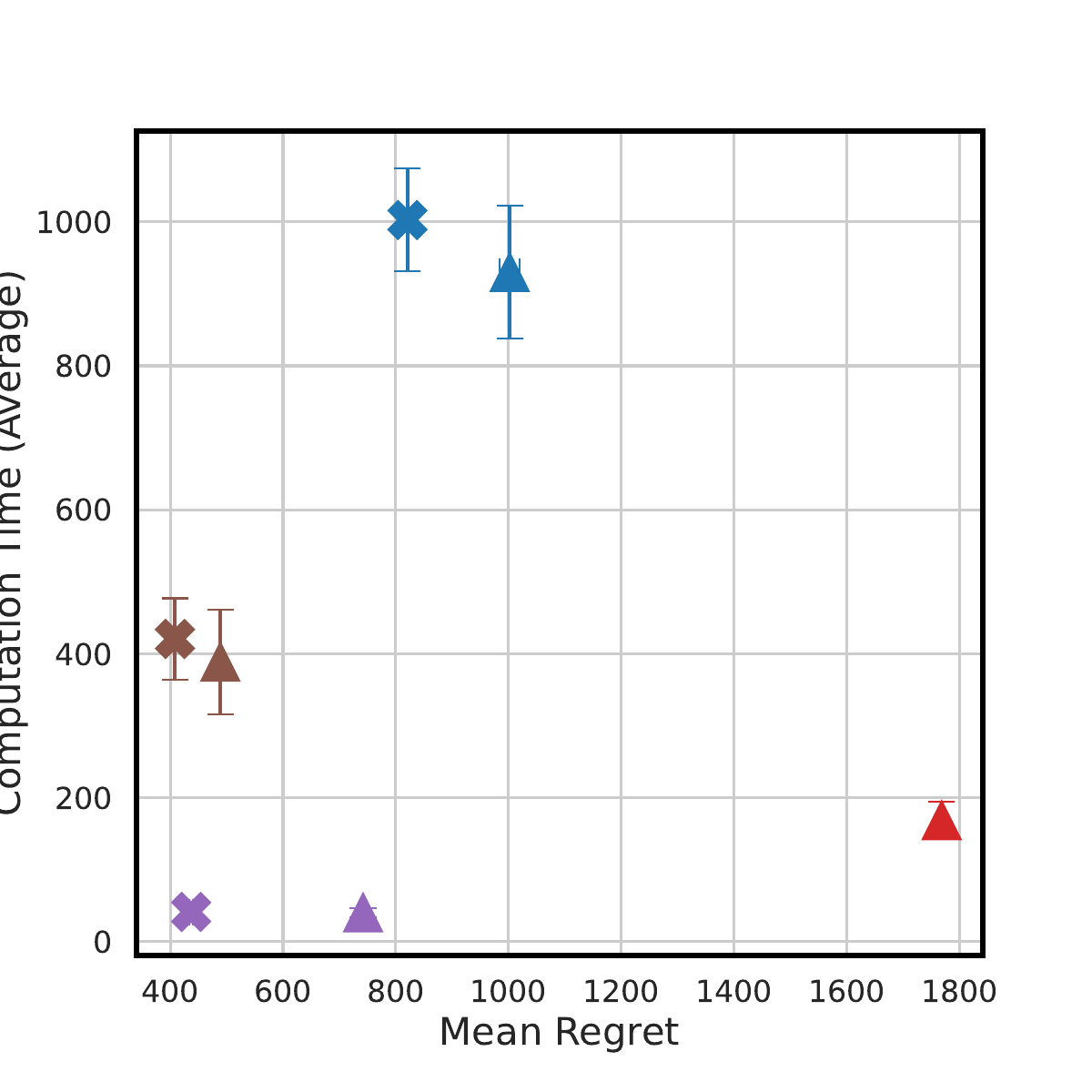}
\end{subfigure}
\caption{Regret and computation times.}
\label{fig:Computation time and regret}
\end{wrapfigure}

In this paper, we answer the above question through the lens of \textit{diminishing exploration}.
Our key insight is that one could achieve a proper trade-off between the detection delay and the regret incurred by exploration if the amount of exploration is configured to decrease with the \textit{elapsed time since the latest detection}, even without the knowledge of $M$.
Specifically, the main contribution of this work is to address the piecewise stationary multi-armed bandit problem from a minimalist perspective, reflected in several key aspects:\\ \textbf{Conceptual Simplicity of the Algorithm}: The proposed algorithm is conceptually very simple, which only involves equipping an active method with a novel diminishing exploration mechanism. This allows it to be flexibly used alongside {\it any} active method, contributing to a clear design logic for the DE+CD+Base algorithm. \\\textbf{Minimal Knowledge Requirement}: The algorithm operates with minimal knowledge of the environment, requiring no prior information about the number of change points $M$, and still can achieve a nearly-optimal regret bound of $\tilde{\mathcal{O}}(\sqrt{MKT})$ due to its ability to automatically adapt to the environment. This balance of minimal assumption and strong performance highlights the algorithm’s ability to provide excellent regret results numerically.\\ \textbf{Low Complexity}: The algorithm maintains one of the lowest possible complexities, designed without adding extra computational burden. The exploration mechanism is efficiently scheduled, requiring only the determination of when to initiate each exploration phase.

To support our statements, we provide evidence in Figure~\ref{fig:Computation time and regret}, where the mean regret and computation time across different algorithms are compared. It is shown that the proposed diminishing exploration together with M-UCB can achieve the best performance in terms of both regret and complexity. This is an initial observation, and we will discuss it in more detail in Section~\ref{sec:sim}.


This paper is summarized as follows: (i) We revisit piecewise-stationary bandit problems without the knowledge of the number of changes through the lens of diminishing exploration, which is parameter-free, computationally efficient, and compatible with various change detection methods and bandit algorithms. (ii) We provide a general form that allows any change detector to be combined with diminishing exploration and formally show that M-UCB and GLR-UCB equipped with the proposed diminishing exploration scheme under a properly chosen scheduler enjoy $\tilde{\gO}(\sqrt{MKT})$ regret bound. Therefore, the proposed algorithm is nearly optimal in terms of dynamic regret without knowing $M$. (iii) Through extensive simulations, we corroborate the regret performance of the proposed algorithm and show that it outperforms the existing benchmark methods in empirical regret. 

A summary of the algorithms, which will be reviewed in what follows, and their performance, along with the one proposed in this work, is provided in Table~\ref{tab:1}. In this table, we highlight the performance of the proposed M-UCB (DE) as evidence of our claim of being an minimalist approach. As shown in the table, it is a nearly optimal algorithm with low computational complexity. In addition, it requires the least knowledge about the environment, offering versatility and ease of extension to M-UCB (DE)$^{+}$. Last but not least, the performance of GLR-UCB (DE) and its extension GLR-UCB (DE)$^{+}$ are also provided as an example to demonstrate that the proposed DE can be used in conjunction with active methods, other than M-UCB.

\subsection{Related Work}
\label{sec:intro:related}

\textbf{{Piecewise-Stationary Bandits With Knowledge of Number of Changes}}. The existing bandit algorithms for the piecewise-stationary setting could be largely divided into two categories: (i) \textit{Passive methods}: The forgetting mechanism is one widely adopted technique to tackle piecewise stationarity without explicitly detecting the change points. For example, Discounted UCB \citep{kocsis2006discounted} and Sliding-Window UCB \citep{garivier2011upper} are two important variants of UCB-type algorithms with respective forgetting mechanism, and they both have been shown to achieve ${\gO}(K\sqrt{MT}\log T)$ dynamic regret. Moreover, \citep{raj2017taming} propose Discounted Thompson Sampling (DTS), which adapts the discounting technique to the Bayesian setup and enjoys good empirical performance despite the lack of any theoretical guarantee. {Subsequently, \citep{qi2023discounted} provides valuable insights into DTS with theoretical guarantees, demonstrating that the DTS method can achieve $\mathcal{O}(K\sqrt{MT}\log^{2}{T})$ regret.} 
By adapting the methods originally designed for adversarial bandits, RExp \citep{besbes2014stochastic} offers another passive strategy by augmenting the classic Exp3 algorithm \citep{auer2002nonstochastic} with restarts and achieve $\gO((K\log K V_T)^{1/3}T^{2/3})$ regret, where $V_T$ denotes the total variation budget up to $T$.
(ii) \textit{Active methods}: MAB algorithms augmented with a change-point detector have been explored quite extensively. 
One example is the Windowed-Mean Shift Detection (WMD) algorithm \citep{yu2009piecewise}, which offers a generic framework of combining change detectors and standard MAB methods.
That said, WMD is designed specifically for the setting with additional side information about the rewards of the unplayed arms and hence is not applicable to the standard piecewise-stationary setting.
\citet{allesiardo2015exp3} propose Exp3.R, which augments Exp3 with a change detector that resets Exp3 and thereby achieves $\gO(NK\sqrt{T\log T})$ with $N$ denoting the number of changes of the best arm ($N=M$ in the worst case). 
More recently, \citet{liu2018change} propose change-detection based UCB (CD-UCB), which combines UCB method with off-the-shelf change detectors, such as the classic cumulative sum (CUSUM) procedure and Page-Hinkley test. 
\citet{cao2019nearly} propose M-UCB, which augments UCB with a simple change detector based on the estimated mean rewards.
Both CD-UCB and M-UCB could achieve the ${\gO}(\sqrt{KMT \log T})$ regret bound.
On the other hand, similar ideas have also been studied from a Bayesian perspective \citep{mellor2013thompson,alami2017memory}.
However, all the methods described above rely on the assumption that $M$ is known.
Moreover, piecewise-stationary bandit has also been studied in the constrained setting \citep{mukherjee2022safety} and the contextual bandit setting \citep{chen2019new,zhao2020simple}.

\textbf{{Piecewise-Stationary Bandits Without Knowledge of Number of Changes}}. Among the existing works, AdSwitch \citep{auer2019adaptively} and GLR-klUCB \citep{besson2022efficient} are the most relevant to ours as they also obviate the need for the knowledge of $M$. Specifically, AdSwitch achieves ${\gO}(\sqrt{KMT \log T})$ regret via an elimination approach, but at the expense of a high computational complexity incurred by a more sophisticated detection scheme. 
On the other hand, GLR-klUCB offers a more efficient detection framework but could achieve an order-optimal regret bound only in easy problem instances. 
In contrast, our diminishing exploration is meant to acheive optimal regret without the knowledge of $M$ nor strong assumptions on the segment length. 

Recent studies on enhancing the practical effectiveness of change detection-based algorithms have considered addressing significant changes without having a complete restart for every detected change. When the reward distributions evolve while the optimal arm remains stable, a full restart is deemed too conservative. 
For instance, \citep{manegueu2021generalized} proposed a change point algorithm based on empirical gaps between arms. \citep{suk2022tracking} expanded on this by quantifying significant shifts at each step, avoiding reliance on non-stationarity knowledge. They employ a sophisticated method to regularly re-explore suboptimal arms, ensuring optimal guarantees for both piecewise-stationary and variation budget assumptions. \citep{abbasi2023new} also achieved comparable guarantees in scenarios with abrupt changes, albeit with slightly diminished results. \citep{pmlr-v134-wei21b} takes a multi-scale approach and maintains multiple competing instances of the base algorithm, which achieves a good theoretical guarantee.




%% file: sections/2-problem.tex
\section{Problem Formulation}
\label{sec:prelim}

\textbf{Piecewise-Stationary Bandit Environment.} A piecewise-stationary bandit can be described using the tuple 
$\left(\gK,\gT,\left\{f_{k,t}\right\}_{k\in\gK,t\in\gT}\right)$, where $\gK$ represents a set of $K$ arms, $\gT$ represents a set of $T$ time steps, and $f_{k,t}$ represents the reward distribution of arm $k\in\gK$ in time $t\in\gT$. Denote by $X_{k,t}$ the reward provided by the environment at the $t$-th time step if the learner selects the $k$-th arm. This reward is drawn from $f_{k,t}$ independently of the rewards obtained in other time steps $t'\in\gT$. 
At time step $t\in\gT$, the learner selects $A_{t}$, one of the $K$ arms as the action at this time step, and sees a reward $X_{A_{t}, t}$. 

Unlike the stochastic bandit environment, the piecewise-stationary bandit environment has several unknown change points, at which the reward distribution will change. 
Let us define $M$ as the total number of segments in this piecewise-stationary bandit environment. Mathematically, $M$ can be represented as 
\begin{equation}
    M:=1+\sum^{T-1}_{t=1}\1_{\left\{f_{k,t}\neq f_{k,t+1} \mathrm{\ for\ any\ } k \in\gK\right\}},
\end{equation}
where the indicator function $\1_{\left\{f_{k,t}\neq f_{k,t+1} \text{ for any } k \in \gK\right\}}$ represents the occurrence of a change. We denote the time of the $i$-th change point as $\nu_{i}$, for all $i\in \{1,\cdots, M-1\}$ and let $\nu_{0}=0$ and $\nu_M=T$. Moreover, we define $s_{i}:=\nu_{i}-\nu_{i-1}$ as the segment length of each segment $i$.
Within the $i$-th segment, for each $k\in\gK$, the reward distribution $f_{k,t}$ are the same for all $t\in[\nu_{i-1}+1,\nu_i]$; therefore from this point onward, we slightly abuse the notation to simply use $f^{(i)}_{k}$ and $\mu^{(i)}_{k}$ to denote the reward distribution and the corresponding expected value, respectively, for arm $k$ in the $i$-th segment.


\textbf{Regret Minimization.} Similar to \cite{liu2018change,cao2019nearly}, we adopt the \textit{dynamic regret} as the performance 
metric:
\begin{equation}
    \gR\left(T\right):=\sum^{T}_{t=1}\max_{k\in\gK}\E\left[X_{k,t}\right] - \E\left[\sum^{T}_{t=1}X_{A_{t},t}\right].
\end{equation}
In the context of piecewise-stationary bandit, our objective, like in other bandit problems, is to minimize regret. 
Bandit algorithms, such as UCB, are known for solving the tension between exploration and exploitation in stationary bandit problems. The main challenge in piecewise-stationary bandit lies again in the trade-off between exploration and exploitation; however, in the other sense. To illustrate, in each segment, after sufficiently exploring the environment, a traditional bandit algorithm will (perhaps softly) commit to the current known best arm. This commitment is difficult to break in the presence of unnoticed change. Additional exploration can be invested to identify changes for resetting the algorithm, which inevitably introduces additional regret. This gives rise to a new exploration and exploitation trade-off that investigates how much additional exploration should be conducted. Specifically, the more additional exploration, the quicker the changes are detected, leading to a quicker reset of the algorithm. The goal of this work is to solve this tension by proposing a novel exploration mechanism that can strike a perfect balance between the regret incurred by additional exploration and that due from detection delay.


%% file: sections/3-algorithm.tex
\section{The Proposed Framework: Diminishing Exploration }
\label{sec:alg}
In this section, we provide the proposed algorithm in detail. We note that in most of the active 
algorithms for piecewise-stationary bandit problems, such as \cite{yu2009piecewise,liu2018change,
cao2019nearly}, there is a (periodic or stochastic) uniform exploration scheme which spends a constant fraction of time on exploration for detecting potential changes, together with a traditional algorithm that is capable of attaining near-optimal tradeoff for the traditional bandit problem. We referred to this type of algorithms as a change detection 
(CD)-based bandit algorithm. 
Our proposed algorithm is a novel and generic exploration technique, called diminishing exploration, which can be used in conjunction with a CD-based bandit algorithm.


\subsection{Diminishing Exploration}\label{subsec:DE}
\begin{wrapfigure}{R}{0.55\textwidth}
\vspace{-0.8cm}
  \begin{center}
    \includegraphics[width=0.6\textwidth]{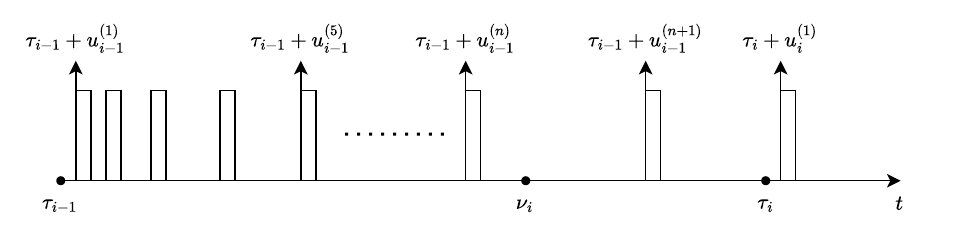}
    \vspace{-20pt}
    \caption{Diminishing exploration.}
    \label{fig:diminishing}
  \end{center}
\end{wrapfigure}
The motivation of the proposed diminishing exploration lies in the following two observations 
about the uniform exploration: 1) The uniform exploration scheme spends a constant fraction of time on exploration, which results in a regret proportional to the configured exploration rate. 2) To determine an exploration rate that achieves the optimal regret scaling, the information about the total number of segments (or change points) is required. To address the above issues, we propose the diminishing exploration scheme as follows.
Let us define $\tau_{i}$ as the $i$-th time when the algorithm alarms a change. In the proposed method, a uniform exploration round starts at $u_{i-1}^{(j)}$ for $j\in\{1, 2, \ldots\}$ with $u_{i-1}^{(1)}=\tau_{i-1}+1$. i.e., the learner chooses to pull the arm $1, 2, \ldots, K$ at time $u_{i-1}^{(j)}, u_{i-1}^{(j)}+1, \ldots, u_{i-1}^{(j)}+K-1$, respectively. The process restarts whenever a new change $\tau_{i}$ is detected. 

We aim to balance the regret resulting from exploration and that associated with the performance of change detection by dynamically adjusting the exploration rate \textit{within a segment}. Let $u_{i-1}^{(j)}$ be the start time of the $j$-th uniform exploration session between two consecutive alarms $\tau_{i}$ and $\tau_{i-1}$. 
In our approach, these sessions are designed in such a way that $u_{i-1}^{(j+1)}-u_{i-1}^{(j)}$ is greater than $u_{i-1}^{(j)}-u_{i-1}^{(j-1)}$. This means that the inter-session time within the same time segment increases with $j$, which in turn results in reduction in the exploration rate. Specifically, for the $i$-th segment, we choose $u_{i}^{(1)}=\left\lceil\left(\alpha-K/4\alpha\right)^{2}\right\rceil$ and
$
    u_{i}^{(j)} = \left\lceil u_{i}^{(j-1)}+\frac{K}{\alpha}\sqrt{u_{i}^{(j-1)}}+\frac{K^{2}}{4\alpha^{2}} \right\rceil,\quad \forall~1\leq i\leq M \textrm,~j\geq 2$, without the knowledge of $M$ and the parameter $\alpha$ will be chosen later. Clearly, we have $u_{i}^{(j-1)}+K < u_{i}^{(j)}$ for every $j\geq 2$; thus, these exploration phases will not overlap. Moreover, it is obvious that the duration between two exploration phases $u_{i}^{(j)}-u_{i}^{(j-1)}=\mathcal{O}(\sqrt{u_{i}^{(j-1)}})$ increases with time as Figure~\ref{fig:diminishing}; hence, the exploration rate decreases. Thus, we term this mechanism {\it diminishing exploration}.

\subsection{Integrating Off-the-Shelf Change Detectors With Diminishing Exploration}
\label{subsec:CD}

The proposed algorithm is given in Algorithm~\ref{alg:main_alg}, which adopts the proposed diminishing exploration (lines 3-8) and executes traditional UCB (lines 9-12) otherwise. Moreover, the algorithm enters the change detection subroutine (lines 17-19) whenever accumulating sufficient observations for an arm (line 18).

Before concluding this section, we reemphasize that although we selected as an example to employ the change detection algorithm in Algorithm~\ref{alg:CD_alg}~\citep{cao2019nearly} and~\ref{alg:glrCD_alg}~\citep{besson2022efficient} in Algorithm~\ref{alg:main_alg} together with the proposed diminishing exploration, the diminishing exploration technique can, in fact, be used in conjunction with any CD-based algorithm.

\begin{wrapfigure}{R}{0.55\textwidth}
\begin{minipage}{0.55\textwidth}
\vspace{-2cm}
\begin{algorithm}[H]
    \caption{CD-UCB with diminishing exploration}\label{alg:main_alg}
    \begin{algorithmic}[1]
        \REQUIRE Positive integer $T,K$ and parameter $\alpha$ 
        \STATE Initialize $\tau\gets 0$, $u\gets \left\lceil\left(\alpha-K/4\alpha\right)^{2}\right\rceil$ and $n_{k}\gets 0\ \forall k \in \mathcal{K}$;
        \FOR{$t = 1,2,\ldots,T$}
            \IF{$u\leq t-\tau <u+K$} 
                 \STATE $A_{t}\gets \left(t-\tau\right)-u+1$ 
            \ELSE
                \IF{$t-\tau = u+K$} 
                    \STATE $u \gets \left\lceil u+\frac{K}{\alpha}\sqrt{u}+\frac{K^{2}}{4\alpha^{2}} \right\rceil $ 
                \ENDIF
                \FOR{$k=1,\ldots,K$}
                    \STATE $\textrm{UCB}_{k}\gets \frac{1}{n_{k}}\sum^{n_{k}}_{n=1}Z_{k,n}+\sqrt{\frac{2\log(t-\tau)}{n_{k}}}$
                \ENDFOR
                \STATE $A_{t}\gets \argmax_{k\in\gK}\textrm{UCB}_{k}$
            \ENDIF
            \STATE Play arm $A_{t}$ and receive the reward $X_{A_{t},t}$.
            \STATE $n_{A_{t}}\gets n_{A_{t}}+1$;$Z_{A_{t},n_{A_{t}}}\gets X_{A_{t},t}$
            \IF{$n_{A_{t}}\geq w$}
                \IF{CD = True}
                    \STATE $\tau\gets t$, $u\gets 1$ and $n_{k} \gets 0\ \forall k \in \mathcal{K}$
                \ENDIF
            \ENDIF
        \ENDFOR
    \end{algorithmic}
\end{algorithm}
\end{minipage}
\end{wrapfigure}

%% file: sections/4-analysis.tex
\section{Regret Analysis}
\label{sec:analysis}

In this section, to show the effectiveness of the proposed diminishing exploration, we define two sets of events to capture the behavior of the change point detection algorithm: The false 
alarm events are defined as $F_{i}:=\left\{\tau_{i}<\nu_{i}\right\},~\forall~1\leq i\leq M-1$, and $F_{0}:=\left\{\tau_{0}=0\right\}$; the event that the detection delay of the 
$i$-th change is smaller than $h_i$ is defined as $D_{i}:=\left\{\tau_{i}\leq \nu_{i}+h_{i}\right\},~\forall~1\leq i \leq M-2$, where the choice of $h_{i}$ depends on the underlying CD algorithm. We also define $D_{0}:=\left\{\tau_{0}=0\right\}$, and  $D_{M-1}:=\left\{\tau_{M-1}\leq T\right\}$.
In our regret analysis, we also define the following two quantities $\Delta^{\left(i\right)}_{k} := \max_{\tilde{k}\in\gK}\left\{\mu^{(i)}_{\tilde{k}}\right\}-\mu^{(i)}_{k},~\forall~1\leq i\leq M,~k\in\gK$, and $\delta^{\left(i\right)}_{k} := \left\lvert \mu^{(i+1)}_{k}-\mu^{(i)}_{k} \right\rvert, \quad \forall~1\leq i\leq M-1,~k\in\gK$. Furthermore, let $\delta^{\left(i\right)} := \max_{k \in \mathcal{K}} \delta^{\left(i\right)}_{k}$.


\begin{theorem}\label{thm:regret}
The \Algref{alg:main_alg} can be combined with a CD algorithm, which achieves the expected regret upper bound as follows:
\vspace{-5mm}
\begin{multline} 
\E\left[R\left(1,T\right)\right]\leq \underbrace{\sum^{M}_{i=1}\tilde{C}_{i}}_{(a)}+\underbrace{2\alpha\sqrt{MT}}_{(b)}
    +\underbrace{\sum^{M-1}_{i=1}\E\left[\tau_{i}-\nu_{i}\middle | D_{i}\overline{F}_{i}D_{i-1}\overline{F}_{i-1}\right]}_{(c)}\\
\hspace{+4cm}+\underbrace{T\sum^{M}_{i=1}\mathbb{P}\left(F_{i}\middle| \overline{F}_{i-1}D_{i-1}\right)+T\sum^{M-1}_{i=1}\mathbb{P}\left(\overline{D}_{i}\middle|\overline{F}_{i}\overline{F}_{i-1}D_{i-1}\right)}_{(d)}, \label{eqn:regret_bound}
\end{multline}
where $\tilde{C}_{i}=8\sum_{\Delta^{\left(i\right)}_{k}>0}\frac{\log T}{\Delta^{\left(i\right)}_{k}}+\left(\frac{5}{2}+\frac{\pi^{2}}{3}+K\right)\sum^{K}_{k=1}\Delta^{\left(i\right)}_{k}$.
\end{theorem}

To elaborate, let us look into each term in \eqref{eqn:regret_bound}. As shown in Lemma~\ref{lemma:regret_stat}, term (a) bounds the regret of the UCB algorithm in each stationary segment, given that the CD algorithm successfully detected the previous change.
Term (b) bounds the regret incurred from the diminishing exploration, as shown in Lemma~\ref{lemma:de_regret}. 
The other two terms bound the regrets incurred in the phase of change detection, whose quantities would depend on the underlying CD algorithm. Specifically, term (c) corresponds to the regret associated with the detection delay while term (d) addresses the regret from unsuccessful detection and false alarm. For a more detailed proof, see Appendix~\ref{app:pf_detail}.


\subsection{\bf Integration with change detectors}\label{sec:analysis:Integration CD}
In this section, we will integrate change detectors from M-UCB and GLR-UCB into the framework of diminishing exploration. Through theoretical analysis, we will demonstrate that diminishing exploration can be extended to other change detectors and achieve a nearly optimal regret bound. All the proofs are deferred to Appendix~\ref{app:pf_detail}.

\vspace{-5pt}
{\bf Integration with the change detector of M-UCB.} 
In M-UCB, change detectors are triggered when the sample count of an arm reaches a window size 
$w$. The change detector divides the samples in the window into two halves and compares the difference between the two halves' summations. If the result exceeds a threshold $b$, an alarm is raised. We define $h_{0}:=0$ and choose $h_{i} = \left\lceil w\left(K/2\alpha+1\right)\sqrt{s_{i}+1}+w^{2}/4\left(K/2\alpha+1\right)^{2} \right\rceil$ and make the following assumption:
\vspace{-4pt}
\begin{assumption}\label{ass:minimum_gap}
    The algorithm knows a lower bound $\delta>0$ such that $\delta\leq\min_{i}\max_{k\in\mathcal{K}}\delta^{\left(i\right)}_{k}$.
\end{assumption}
\vspace{-3pt}
Note that Assumption~\ref{ass:minimum_gap} is Assumption 1(b) of \cite{cao2019nearly}, which is required for the M-UCB detector to determine good $w$ and $b$ in regret analysis. It is worth noting that almost all schemes that actively detect changes share similar assumptions; however, different algorithms may impose distinct sets of assumptions. This assumption is mild since $\delta$ may be statistically derived from historical information. Furthermore, even if the lower bound does not hold, and we occasionally encounter changes with expected reward gaps smaller than the assumed $\delta$, such changes may be perceived as too minor to result in significant regret. In Section~\ref{sec:sim}, this fact will be verified through simulation.

With this assumption, we analyze the regret of Algorithm~\ref{alg:CD_alg} with $w$ and $b$ given by
\begin{align}
    w&=\left(4/\delta^{2}\right)\cdot \left[\sqrt{\log\left(2KT^{2}\right)}+\sqrt{\log\left(2T\right)}\right]^{2},\label{eqn:w_fix}\\
    b&=\left[w\log\left(2KT^{2}/2\right)\right]^{1/2}.\label{eqn:b_fix}
\end{align}
\begin{assumption}\label{asm:seg_length}
    $s_{i}= \Omega\left(\left(\log{KT}+\sqrt{K\log{KT}}\right)\sqrt{s_{i-1}}\right)$. 
\end{assumption}
\vspace{-5pt}
In particular, if $s_i = \Theta\left(\left(\log{KT}+\sqrt{K\log{KT}}\right)^{2(1+\epsilon)}\right)$ for every $i$, Assumption~\ref{asm:seg_length} holds.
This assumption essentially posits that the changes are not overly dense, a condition that holds in many practical scenarios. Simple math would then show that given $D_{i-1}$ is true, with this assumption and the proposed diminishing exploration, each arm will have at least $w/2$ observations before and after a change point. Again, we note that similar assumptions are imposed in other algorithms that actively detect changes, while different algorithms may impose different assumptions. This assumption is necessary with our proof technique, which requires every change to be successfully detected with high probability. In our simulations in Section~\ref{sec:sim}, we will demonstrate that when this assumption is violated, all the considered active methods will experience similar performance degradation due to the overly dense changes and our algorithm is not particularly vulnerable. In fact, in our simulation results in Section~\ref{sec:sim}, we show that our algorithms significantly outperform existing active methods, even when this assumption is violated.

\begin{corollary}[Regret bound of M-UCB]\label{cor:regret_MUCB}
    Algorithm~\ref{alg:main_alg} integrated with Algorithm ~\ref{alg:CD_alg} with the parameters in (\ref{eqn:w_fix}) and (\ref{eqn:b_fix}) achieves the expected regret upper bound of $\mathcal{O}(\sqrt{KMT\log{T}})$.
\end{corollary}
\vspace{-5pt}
{\bf Integration with the change detector of GLR-UCB.} In GLR-UCB, the Generalized Likelihood Ratio (GLR) test is employed on the samples to detect changes. i.e., an alarm is raised whenever the GLR statistic exceeds a threshold $\beta$ given by
\begin{equation}
    \beta=2\mathcal{J}\left(\frac{\log{(3T^2)}}{2}\right)+6\log{(1+\log{T})}, \label{eqn:beta}
\end{equation}
where the function $\mathcal{J}$ is defined in Appendix~\ref{app:pf_detail:not_extend:GLRUCB}. Following \citep{besson2022efficient}, we define $h_0:=0$ and choose $h_{i} =\left(\alpha,\epsilon\right) :=\left\lceil 2\left(\frac{4}{\left(\delta^{\left(i\right)}\right)^{2}}\beta+2\right)\left(\frac{K}{2\alpha}+1\right)\sqrt{s_{i}+1}+\left(\frac{4}{\left(\delta^{\left(i\right)}\right)^{2}}\beta+2\right)^{2}\left(\frac{K}{2\alpha}+1\right)^{2} \right\rceil$ and make the following:
    
\begin{assumption}\label{ass:glr_delay_context}
    $\nu_{i}-\nu_{i-1}\geq 2\max\left\{h_{i}, h_{i-1}\right\}$ for all $i\in\left\{1,\ldots,M\right\}$.
\end{assumption}
\vspace{-5pt}
Under this assumption, we prove the following:
\begin{corollary}[Regret bound of GLR-UCB]\label{cor:regret_glrUCB}
    Algorithm~\ref{alg:main_alg} integrated with Algorithm~\ref{alg:glrCD_alg} with $\beta$ in (\ref{eqn:beta}) achieves the expected regret upper bound as $\mathcal{O}(\sqrt{KMT\log{T}})$.
\end{corollary}

%% file: sections/5-discuss.tex
\vspace{-8pt}
{\bf Discussion.} In some literature, such as \cite{liu2018change} and \cite{cao2019nearly}, the approach involves finding an exploration rate for uniform exploration that balances regret induced by exploration and detection delay, assuming knowledge of $M$. In GLR-UCB \cite{besson2022efficient}, the exploration rate increases with the number of change detection alarms generated by the algorithm. Compared to other exploration mechanisms, the distinctive feature of {\em diminishing exploration} is its use of a variable exploration rate within the same segment. The greatest advantage of this approach lies in the fact that it does not require knowledge of $M$. 
Moreover, its complexity remains low, and it can be readily applied to other active methods.



%% file: sections/5-extension.tex
\section{Extension to Detection of Optimal Arm Changes}
\label{sec:extension}
Let us define $S$ as the number of \textit{super-segments}, each of which is the time period between two consecutive changes of the optimal arm. Mathematically, $S$ can be represented as 
\begin{equation}
    S:=1+\sum^{T-1}_{t=1}\1_{\left\{\argmax_{k\in\gK}\mu_{k,t}\neq 
    \argmax_{k\in\gK}\mu_{k,t+1}\right\}}
\end{equation}
where the indicator function $\1_{\left\{\argmax_{k\in\gK}\mu_{k,t}\neq \argmax_{k\in\gK}\mu_{k,t+1}\right\}}$ represents the occurrence of the optimal arm changing to another one. We denote the time of the $r$-th occurrence of the optimal arm changing to another one as $\nu_{r}^{*}$, for all $r\in\left\{1,\cdots,S-1\right\}$ and let $\nu_{r}^{*}=0$ and $\nu_{S}^{*}=T$. Moreover, we define $s_{r}^{*} := \nu_{r}^{*} - \nu_{r-1}^{*}$ as the super segment length of each segment $r$, which is the duration for which the optimal arm remains the same.
The last, we define $\tau_{r}^{*}$ as the $r$-th time when the algorithm alarms an optimal arm changing to another one. We have provided Figure~\ref{fig:notation} to visually clarify the differences from Section~\ref{sec:analysis}.

\begin{wrapfigure}{r}{0.4\textwidth}
\begin{minipage}{0.4\textwidth}
\vspace{-0.8cm}
\begin{algorithm}[H]
    \caption{Skipping Mechanism}\label{alg:skip}
    \begin{algorithmic}[1]
        \REQUIRE Two positive integer $n_{k}$ and $n_{k^{*}}$, $n_{k}$ observations $X_{1},\ldots,X_{n_{k}}$ and $n_{k^{*}}$ observations $Y_{1},\ldots,Y_{n_{k^{*}}}$.
        \IF{$\sum^{n_{k}}_{\ell=1}X_{\ell}/n_{k}<\sum^{n_{k^{*}}}_{\ell=1}Y_{\ell}/n_{k^{*}}+\eta$}
        \STATE Return True
        \ELSE
        \STATE Return False
        \ENDIF
    \end{algorithmic}
\end{algorithm}
\end{minipage}
\end{wrapfigure}

Similar to Section~\ref{sec:analysis}, we also define two sets of events to capture the behavior of the change point detection algorithm in the version where we only focus on the change of the optimal arm: The false 
alarm events are defined as $F_{r}^{*}:=\left\{\tau_{r}^{*}<\nu_{r}^{*}\right\},~\forall~1\leq r\leq S-1$, and $F_{0}^{*}:=\left\{\tau_{0}^{*}=0\right\}$; the event that the detection delay of the 
$r$-th change is smaller than $h_{r}^{*}$ is defined as $D_{r}^{*}:=\left\{\tau_{r}^{*}\leq \nu_{r}^{*}+h_{r}^{*}\right\},~\forall~1\leq r \leq S-2$, where the choice of $h_{r}^{*}$ depends on the CD algorithm. We also define $D_{0}^{*}:=\left\{\tau_{0}^{*}=0\right\}$, and $D_{S-1}^{*}:=\left\{\tau_{S-1}^{*}\leq T\right\}$. 
In our regret analysis, we also define the following quantities $\Delta_{k,t}:=\max_{\tilde{k}\in\gK}\left\{\mu_{\tilde{k},t}\right\}-\mu_{k,t},~\forall~1\leq t\leq T,~k\in\gK$, and assume $\Delta_{\min}:=\min_{k\in\gK}\min_{1\leq t\leq T}\Delta_{k,t}$ is known.

\begin{remark}
    In this section, false alarms differ from Section~\ref{sec:analysis}. When the change detection algorithm declares an alarm, even if a change has occurred but the optimal arm remains the same, this situation will also be considered a false alarm.
\end{remark}
\subsection{Diminishing Exploration With a Skipping Mechanism}
\label{sec:extension:mechanism}
\vspace{-5pt}
Here, we introduce a skipping mechanism (Algorithm~\ref{alg:skip}) to ignore unnecessary alarms. The algorithm takes two sets of observations of size $n_k$ and $n_k^*$, respectively, as inputs and checks whether the sample average of the second set is larger than that of the first set by a margin of $\eta$. In Appendix~\ref{sec:skip_whole_algorithm}, we present the complete algorithm, where the two sets are samples of our algorithm before and after an alarm of change, respectively. If Algorithm~\ref{alg:skip} returns true, then this alarm is skipped; otherwise, it declares that the optimal arm has changed and resets the algorithm. Note that having a negative $\eta$ would reduce miss detection but may also increase false reset. On the other hand, having a positive $\eta$ would encourage skipping, reducing false reset but leading to higher miss detection. Besides, the optimal $\eta$ also depends heavily on the underlying CD algorithm, as some CD algorithms cause higher false alarm rates than others. In our numerical (in Appendix~\ref{app:sim}) and analytic results (Theorem~\ref{thm:regret} in Appendix~\ref{app:pf_detail:extend}), we set $n_k = \mathcal{O}(\log T)$ and $n_k^* = \mathcal{O}(\log T)$ and demonstrate the effectiveness of the proposed skipping mechanism with $\eta=0$.


\vspace{-5pt}
\subsection{Integration with change detectors}
\label{sec:extension:}
\vspace{-5pt}
In this section, we will integrate change detectors from M-UCB and GLR-klUCB into the extension of diminishing exploration similar to Section~\ref{sec:analysis}.\\
{\bf Integration with change detectors of M-UCB.} We choose the parameter $w$ as
\begin{equation}\label{eqn:w_fix2}
    w=\left(8/\min\left\{\delta,\Delta_{min}\right\}^{2}\right)\cdot \left[\sqrt{\log\left(2KT^{2}\right)}+\sqrt{\log\left(2T\right)}\right]^{2}.
\end{equation}
The selection of the remaining parameters is the same as in Section~\ref{sec:analysis}.
\vspace{-5pt}
\begin{corollary}[Regret bound of M-UCB]\label{cor:regret_bound_extend_MUCB}
    Combining Algorithm~\ref{alg:main_alg} and~\ref{alg:CD_alg} with the parameters in Equation~\ref{eqn:b_fix}, and Equation~\ref{eqn:w_fix2} achieves the expected regret upper bound as $\mathcal{O}(\sqrt{KST\log{T}})$.
\end{corollary}
\vspace{-5pt}
{\bf Integration with change detectors of GLR-UCB.} The selection of parameters is the same as in Section~\ref{sec:analysis}.
\vspace{-5pt}
\begin{corollary}[Regret bound of GLR-UCB]\label{cor:regret_bound_extend_glrUCB}
    Combining Algorithm ~\ref{alg:main_alg} and ~\ref{alg:glrCD_alg} with $\beta$ function in Equation~\ref{eqn:beta} achieves the expected regret upper bound as $\mathcal{O}(\sqrt{KST\log{T}})$.
\end{corollary}
\vspace{-5pt}
The proofs for Corollary~\ref{cor:regret_bound_extend_MUCB} and~\ref{cor:regret_bound_extend_glrUCB} are in Appendix~\ref{app:pf_detail:extend}.

%% file: sections/6-simulate.tex
\section{Simulation Results}
\label{sec:sim}
In this section, we assess the effectiveness of the proposed diminishing exploration scheme across various dimensions, encompassing regret scaling in $M$, $K$, and $T$, regrets in synthetic environments, and regrets in a real-world scenario. In addition to evaluating M-UCB \citep{cao2019nearly} with our diminishing exploration, we also examine a variant of CUSUM-UCB \citep{liu2018change} that incorporates diminishing exploration with CUSUM-UCB, further highlighting the efficacy of the proposed exploration method. We will compare our approach with M-UCB, CUSUM-UCB, GLR-UCB, Discounted-UCB, and Discounted Thompson Sampling. 
Unless stated otherwise, we report the average regrets over 100 simulation trials. Detailed configuration is provided in Appendix~\ref{app:para}.
\vspace{-2.2cm}
\begin{figure}[H]
\centering
\begin{subfigure}{0.85\textwidth}
    \includegraphics[width=\textwidth]{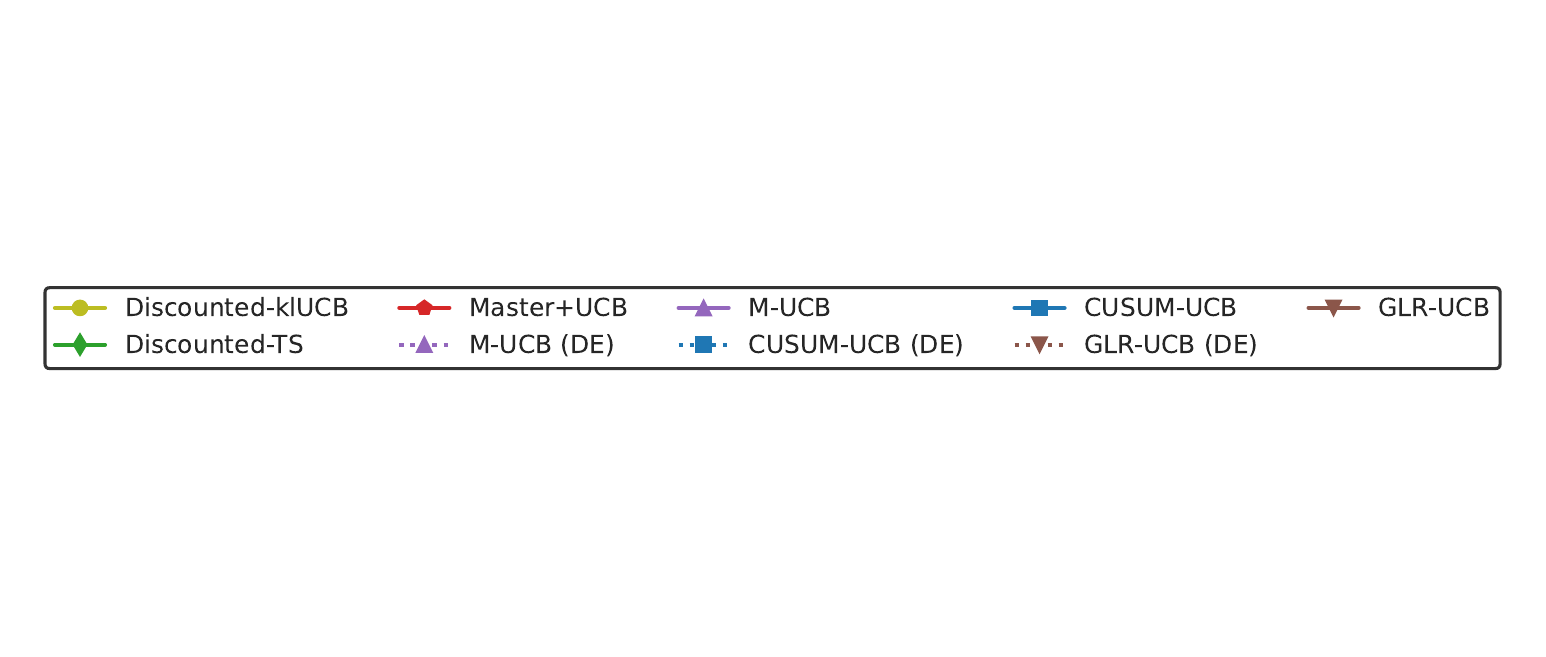}
    \vspace{-2cm}

\end{subfigure}
\begin{subfigure}{0.24\textwidth}
    \includegraphics[width=\textwidth]{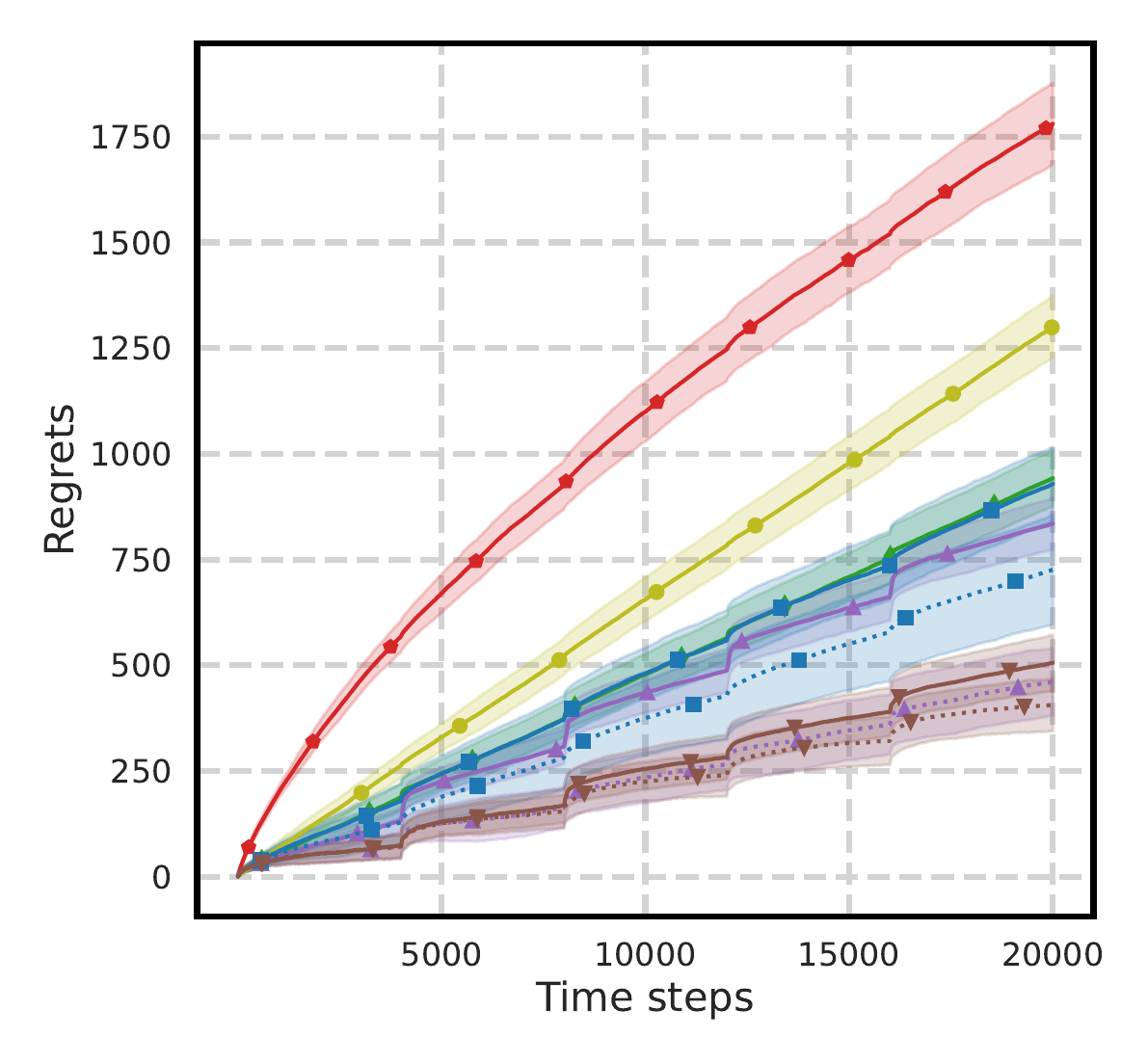}
    \vspace{-15pt}
    \caption{Scaling in $t$.}
    \label{fig:t}
\end{subfigure}
\hspace{-6pt}
\begin{subfigure}{0.24\textwidth}
    \includegraphics[width=\textwidth]{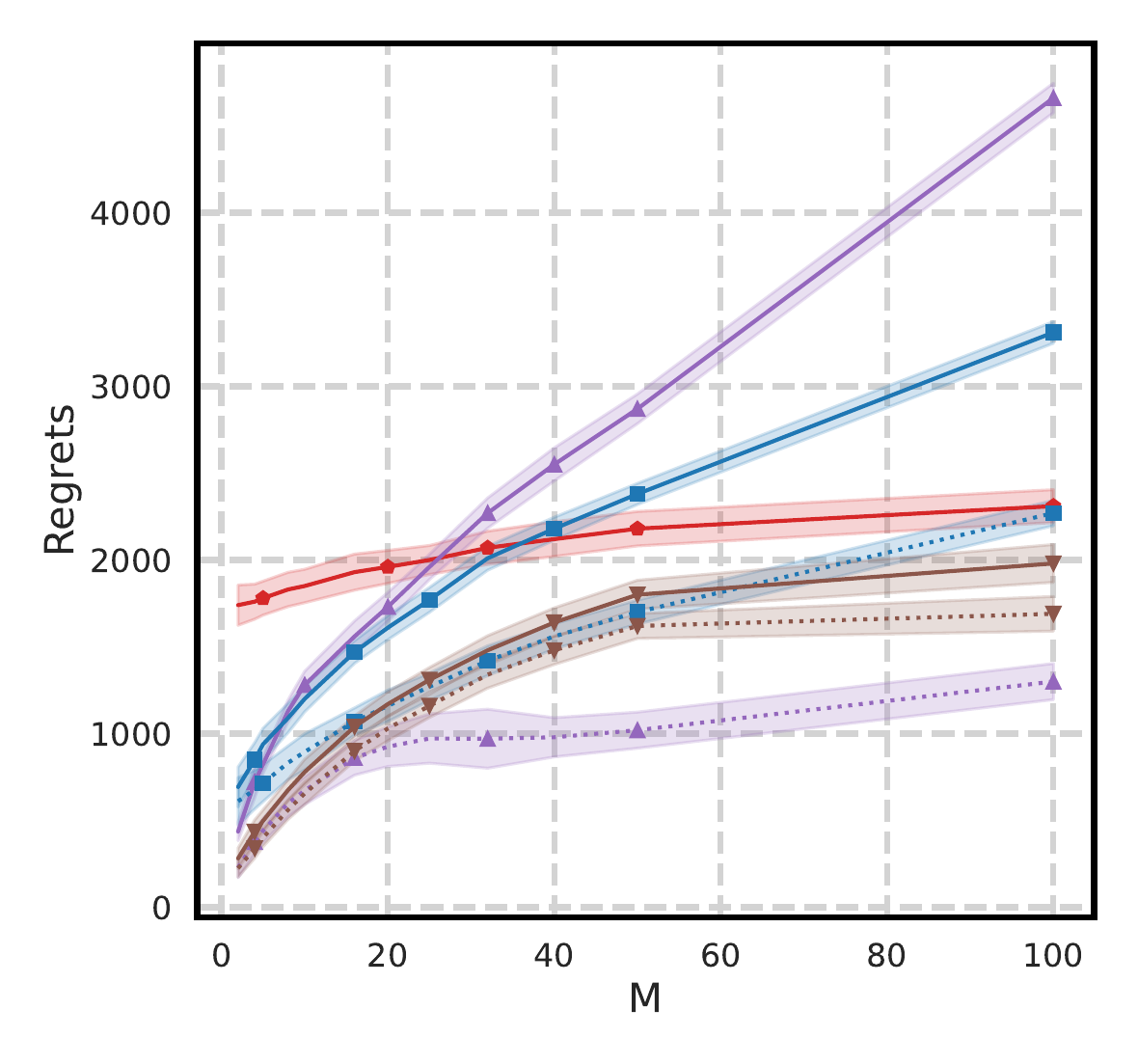}
    \vspace{-15pt}
    \caption{Scaling in $M$.}
    \label{fig:M}
\end{subfigure}
\hspace{-6pt}
\begin{subfigure}{0.24\textwidth}
    \includegraphics[width=\textwidth]{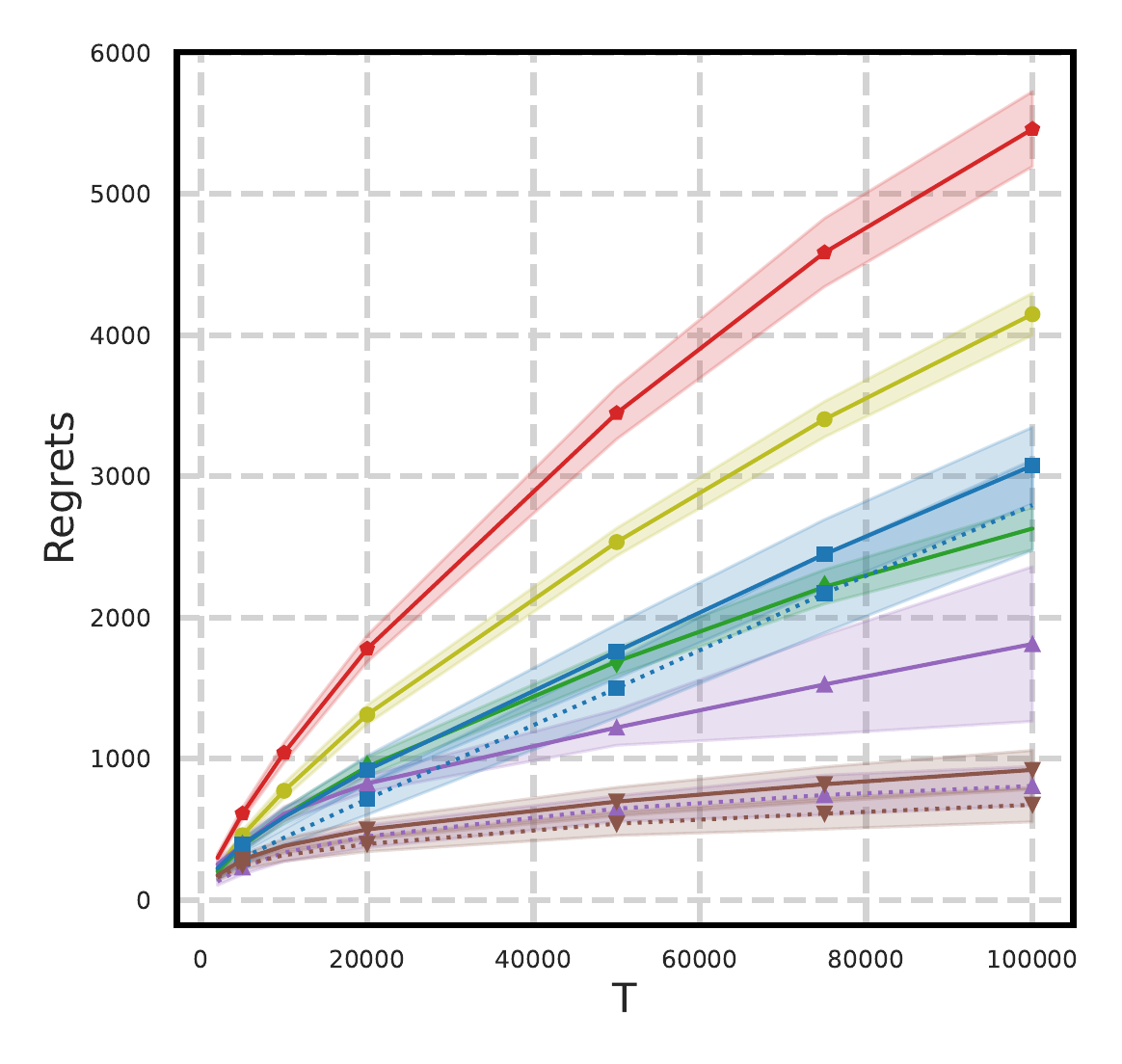}
    \vspace{-15pt}
    \caption{Scaling in $T$.}
    \label{fig:T}
\end{subfigure}
\hspace{-6pt}
\begin{subfigure}{0.24\textwidth}
    \includegraphics[width=\textwidth]{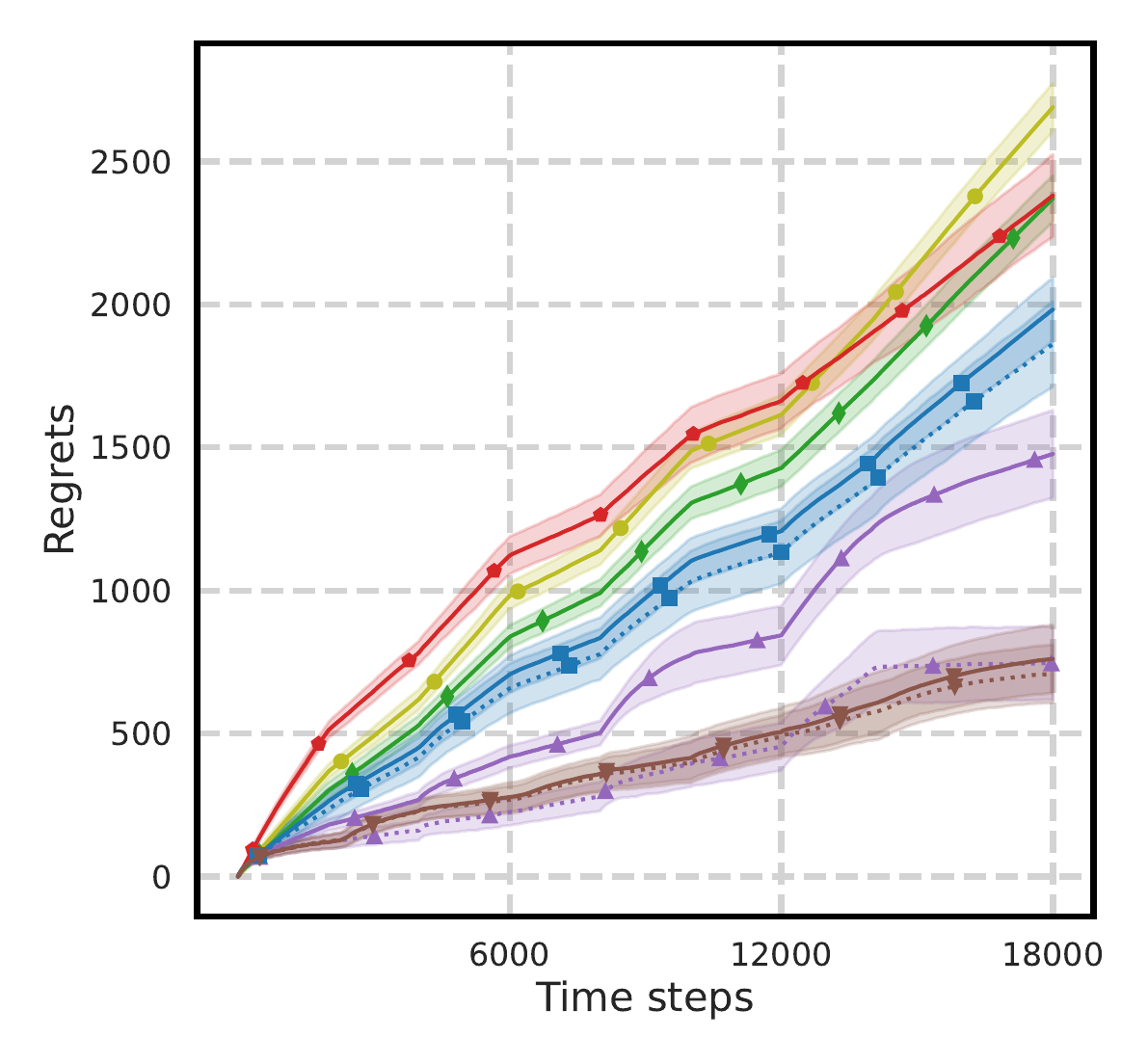}
    \vspace{-15pt}
    \caption{Yahoo Data Set.}
    \label{fig:Yahoo}
\end{subfigure}
\hspace{-10pt}
\caption{Regret in the synthetic environments and under the Yahoo data set.}
\end{figure}
\vspace{-0.5cm}
{\bf Regret in Each Time Step.} In this simulation, we consider a multi-armed bandit problem with $T = 20000$ time steps and $M=5$. 
Recall that $\mu^{(i)}_{k}$ represents the expected value for arm $k$ in the $i$-th segment. Here, we set $\mu^{(i)}_{k}=0.2,0.5,0.8$ for $i$ with $(i+k)\bmod 3=2, 0, 1$, respectively.
Figure~\ref{fig:t} shows that for both CUSUM-UCB and M-UCB, employing diminishing exploration can effectively reduce the additional regret caused by constant exploration. 
Moreover, M-UCB with the proposed diminishing exploration achieves the lowest regret.
In the figure, the change points are clearly evident by observing the breakpoints in each line. The reason for the overall steeper slope of CUSUM-UCB (both with and without diminishing exploration) is due to the heightened sensitivity of the CUSUM detector itself, resulting in more frequent false alarms. 


{\bf Regret Scaling in $M$.} Based on the settings outlined above, with the only variation being in the parameter $M$, Figure~\ref{fig:M} illustrates the dynamic regrets across various values of $M$. In this experiment, adjustments to the exploration parameter settings are required based on the size of $M$ when using a constant exploration rate. However, this is not the case for the proposed diminishing exploration. The result confirms the earlier-discussed rationale that the proposed diminishing exploration can automatically adapt to the environment, resulting in the best regret performance among the algorithms. 

{\bf Regret Scaling in $T$.} In line with the setting described above, with the only variation being the parameter $T$, we present the dynamic regrets across different values of $T$. Observations similar to those made above can again be found in Figure~\ref{fig:T}, where the proposed diminishing exploration can effectively reduce the regret.

\textbf{Regret and Execution Time.} Here, we compare the computation time and regret across different algorithms for various choices of $M$ and $T$ with other parameters same as above. Specifically, we conduct experiments under three scenarioso: one where the environment changes rapidly ($M=50$ and $T=20000$), one where the environment changes slowly ($M=5$ and $T=20000$), and one where the considered time horizon is double ($M=5$ and $T=40000$). As shown in Figure~\ref{fig:ComM50} to \ref{fig:ComBigT}, despite oblivious of $M$, our algorithm almost always achieves the lowest computation time and regret in all the scenarios. Moreover, comparing to Master+UCB, another algorithm not requiring the knowledge of $M$, our algorithm is always significantly better as shown in these figures. Figure~\ref{fig:ratio} plots the ratio of average execution time of Master+UCB to that of M-UCB with our DE for various $T$. It is shown that the growth rate is faster than $0.5 \log T$, and it appears to become even linear in $T$ as $T$ increases.

\vspace{-2.5cm}
\begin{figure}[H]
\centering
\begin{subfigure}{0.85\textwidth}
    \includegraphics[width=\textwidth]{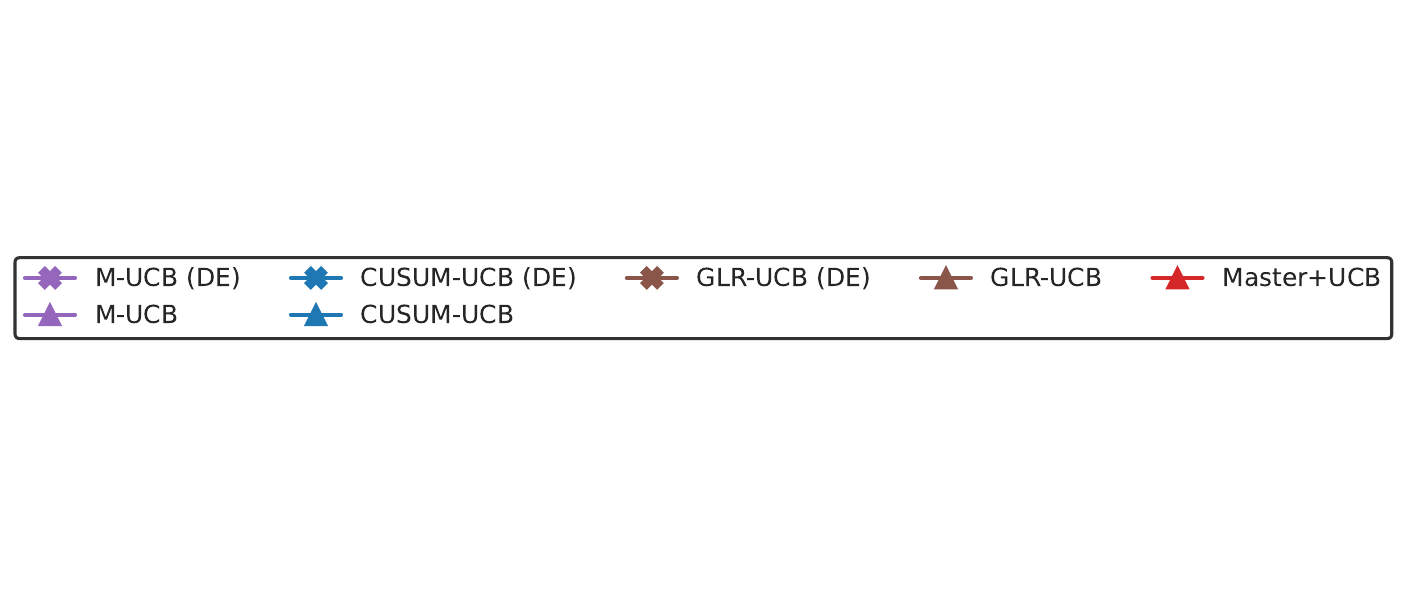}
    \vspace{-2.5cm}
\end{subfigure}
\hfill
\begin{subfigure}{0.25\textwidth}
    \includegraphics[width=\textwidth]{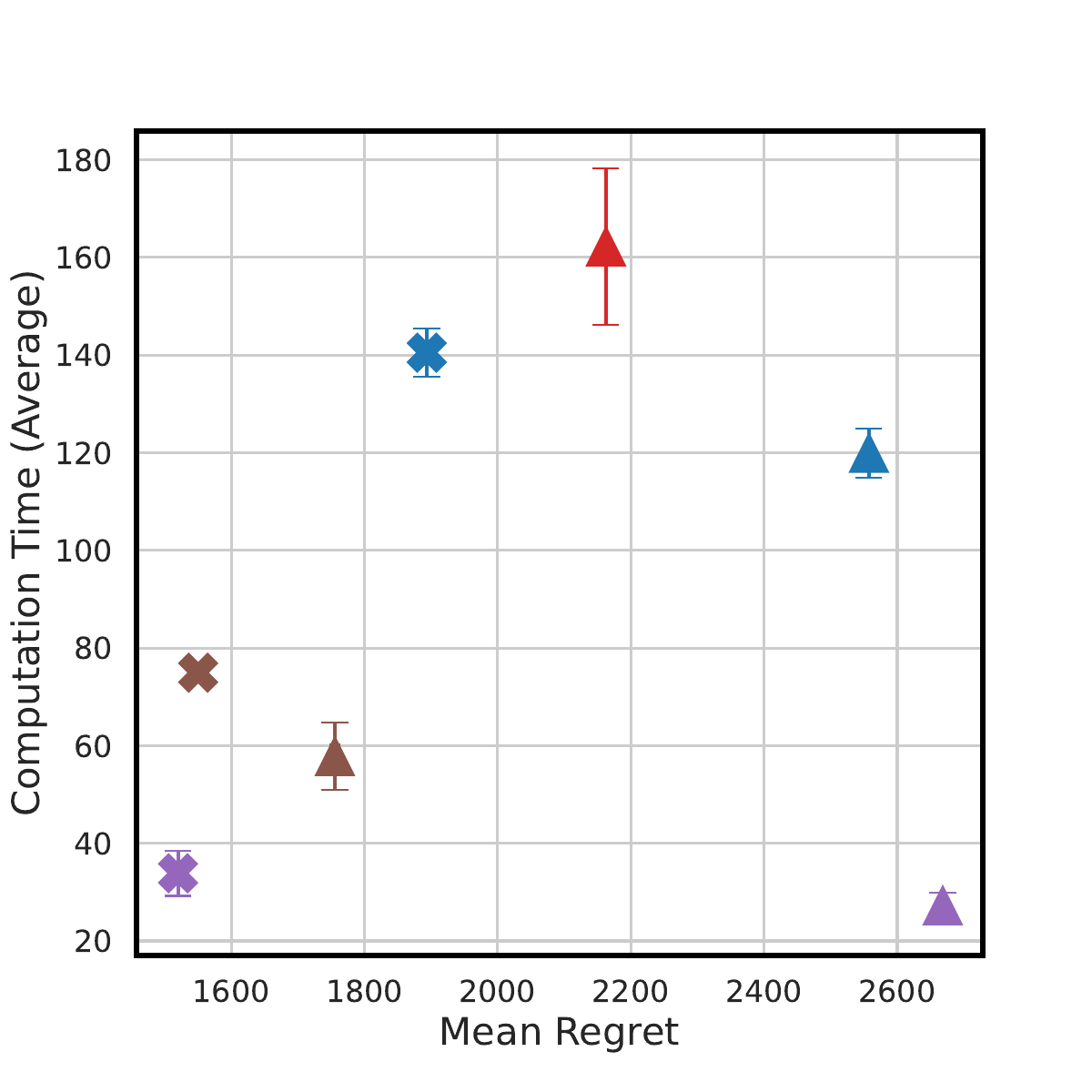}
    \vspace{-15pt}
    \caption{$M=50$, $T=20000$.}
    \label{fig:ComM50}
\end{subfigure}
\hspace{-20pt}
\hfill
\begin{subfigure}{0.25\textwidth}
    \includegraphics[width=\textwidth]{M5_Computation_plot.pdf}
    \vspace{-15pt}
    \caption{$M=5$, $T=20000$.}
    \label{fig:ComM5}
\end{subfigure}
\hspace{-20pt}
\hfill
\begin{subfigure}{0.25\textwidth}
    \includegraphics[width=\textwidth]{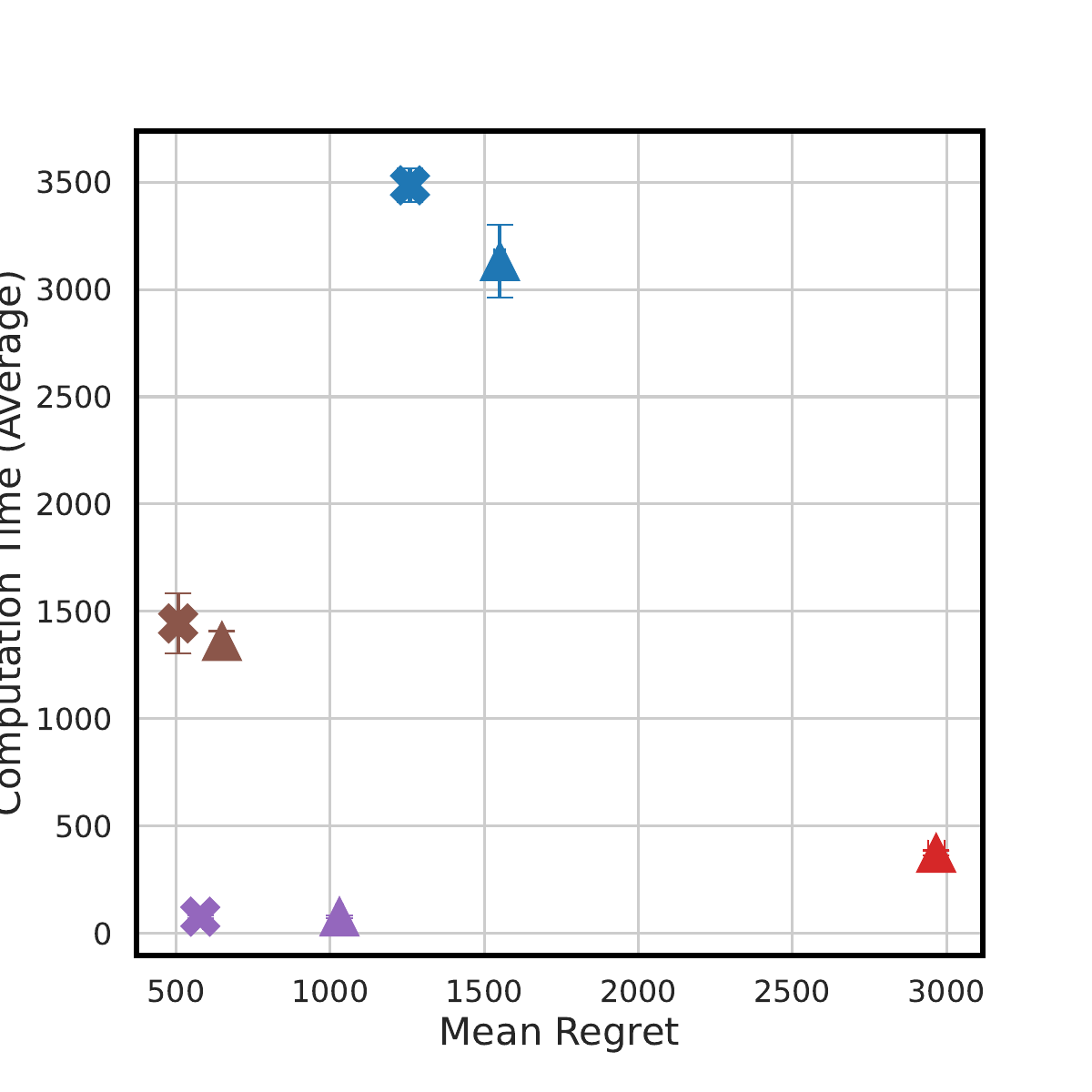}
    \vspace{-15pt}
    \caption{$M=5$, $T=40000$.}
    \label{fig:ComBigT}
\end{subfigure}
\hspace{-20pt}
\hfill
\begin{subfigure}{0.25\textwidth}
    \includegraphics[width=\textwidth]{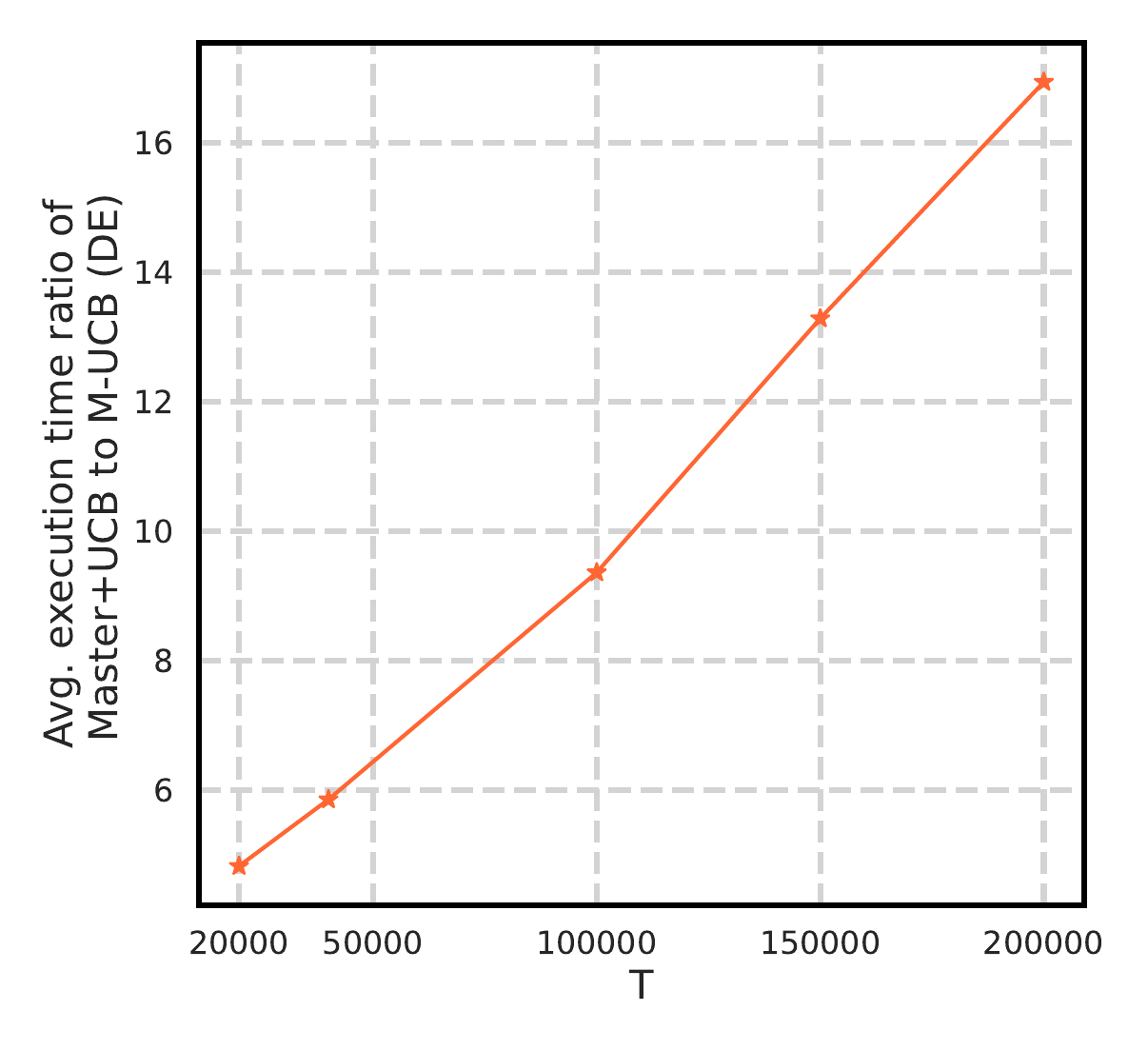}
    \vspace{-15pt}
    \caption{Ratio of avg. execution time of Master+UCB to that of M-UCB (DE).}
    \vspace{-20pt}
    \label{fig:ratio}
\end{subfigure}
\hspace{-20pt}
\caption{Regret and computation times.}
\end{figure}

\textbf{Regret in an Environment Built from a Real-World Dataset.} We further utilize the benchmark dataset publicly published by Yahoo! for evaluation. To enhance arm distinguishability in our simulation, we scale up the data by a factor of $10$. The number of segments is set to $M=9$ and the number of arms is set to $K=6$. Figure \ref{fig:Yahoo} shows the evolution of dynamic regret. Again, we see that the diminishing exploration scheme could help M-UCB, GLR-UCB and CUSUM-UCB achieve similar or better regret even without knowing $M$.

\textbf{Scenarios when Assumptions are Violated.} In Assumption~\ref{ass:minimum_gap}, we assumes the knowledge of $\delta$ to select an appropriate $w$. We emphasize that our settings in many of the above simulations actually violate this assumption. Take Figure~\ref{fig:t} for example, $w=200$ is chosen, which corresponds to $\delta\approx 0.6$ when back calculating, which is much larger than the actual $\delta=0.3$ of this scenario. Assumptions~\ref{asm:seg_length} and~\ref{ass:glr_delay_context} provide guarantees that the segment length is sufficiently long. However, in the last data point of our Figure~\ref{fig:M} ($M=100$), these assumptions are clearly violated. In this case, it becomes challenging for the active methods to promptly detect every change point. For algorithms like M-UCB and CUSUM-UCB, which require knowledge of $M$, their awareness of quick changes causes them to invest more effort into change detection, leading to a very high exploration rate. However, this does not always guarantee successful detection, resulting in very high regret. Diminishing exploration, on the other hand, continues to decrease the exploration rate regardless of $M$ when no changes are detected. This allows more resources to be invested for UCB, which might gradually adapt to the environment's changes, leading to a regret that is lower than that of uniform exploration.


Regarding the simulations of AdSwitch, ArmSwitch, and the Meta algorithm, due to the extremely long running time, we have not been able to finish the simulation for $T$ beyond 20000. We alternatively perform comparison with smaller $T$, whose results are presented in Appendix~\ref{app:sim}.

%% file: sections/7-conclude.tex
\section{Concluding Remarks}
\label{sec:conclusion}

In this paper, we revisited the piecewise-stationary bandit problem. A novel diminishing exploration mechanism, called diminishing exploration, was proposed that does not require knowledge about the number of stationary segments. When used in conjunction with the M-UCB and GLR-UCB, the proposed diminishing exploration mechanism was rigorously shown to achieve a near optimal regret. Extensive simulations were also provided to show the effectiveness of the proposed mechanism. 
Regarding the limitations, since the proposed diminishing exploration is employed together with a CD algorithm, the integrated algorithm usually inherits from the CD algorithm a constraint on the length of each segment in order to guarantee near-optimal regret performance. 

%% file: appendix/1-changealg.tex
\section{Change Detection Algorithm}\label{app:cd_alg}
{\bf Change Detector of M-UCB } \\
The following algorithm is the change detection algorithm for M-UCB \citep{cao2019nearly}
\begin{center}
    \begin{minipage}{0.8\linewidth}
        \begin{algorithm}[H]
            \caption{Change Detection of M-UCB: CD$\left(w,b,Z_{1},\ldots,Z_{w}\right)$}\label{alg:CD_alg}
            \begin{algorithmic}[1]
                \REQUIRE An even integer $w$, $w$ observations $Z_{1},\ldots,Z_{w}$, and a prescribed threshold $b > 0$
                \IF{$\left\lvert\sum^{w}_{\ell=w/2+1}Z_{\ell}-\sum^{w/2}_{\ell=1}Z_{\ell}\right\rvert>b$}
                \STATE Return True
                \ELSE
                \STATE Return False
                \ENDIF
        \end{algorithmic}
    \end{algorithm}
\end{minipage}

\end{center}

In this algorithm, one requires $w$ observations as input and check whether the difference between the sample average of the first half and that of the second half exceeds a prescribed threshold $b$ (line 1).

{\bf Change Detector of GLR-UCB }\\
The following definition is the change detection algorithm for GLR-UCB \citep{besson2022efficient}

\begin{definition}
    The Bernoulli GLR change-point detector with threshold function $\beta\left(n,\epsilon\right)$ is
    \begin{equation}
        \tau_{\delta}:=\inf\left\{n\in\mathbb{N}^{*}:\sup_{s\in\left[1,n\right]}\left[s\times \textrm{kl} \left(\hat{\mu}_{1:s},\hat{\mu}_{1:n}\right)+\left(n-s\right)\times \textrm{kl} \left(\hat{\mu}_{s+1:n},\hat{\mu}_{1:n}\right)\right]\geq \beta\left(n,\delta\right)\right\}
    \end{equation}
\end{definition}

\begin{center}
    \begin{minipage}{\linewidth}
        \begin{algorithm}[H]
            \caption{Change Detection of GLR-UCB: CD$\left(Y_{1},\ldots,Y_{n_{k}}\right)$}\label{alg:glrCD_alg}
            \begin{algorithmic}[1]
                \REQUIRE $Y_{1},\ldots,Y_{n_{k}}$, and a threshold function $\beta(n,\epsilon)$
                \IF{$\sup_{s\in\left[1,n\right]}\left[s\times \textrm{kl} \left(\sum^{s}_{\ell=1}Y_{\ell}/s,\sum^{n}_{\ell=1}Y_{\ell}/n\right)+\left(n-s\right)\times \textrm{kl} \left(\sum^{n}_{\ell=s+1}Y_{\ell}/\left(n-s\right),\sum^{n}_{\ell=1}Y_{\ell}/n\right)\right]\geq \beta\left(n,\delta\right)$}
                \STATE Return True
                \ELSE
                \STATE Return False
                \ENDIF
            \end{algorithmic}
        \end{algorithm}
    \end{minipage}

\end{center}
    

%% file: appendix/extendalg.tex
\newpage
\section{The Extended Version Algorithm}\label{sec:skip_whole_algorithm}

\begin{algorithm}[!h]
    \caption{CD-UCB with diminishing exploration}\label{alg:main_alg_extend}
    \begin{algorithmic}[1]
        \REQUIRE Positive integer $T,K$, and parameter $\alpha$, $N_{I}$ 
        \STATE Initialize $\tau\gets 0$, $u\gets \left\lceil\left(\alpha-K/4\alpha\right)^{2}\right\rceil$ and $n_{k}\gets 0\ \forall k \in \mathcal{K}$;
        \FOR{$t = 1,2,\ldots,T$}
            \IF{$u\leq t-\tau <u+K$} 
                 \STATE $A_{t}\gets \left(t-\tau\right)-u+1$ 
            \ELSE
                \IF{$t-\tau = u+K$} 
                    \STATE $u \gets \left\lceil u+\frac{K}{\alpha}\sqrt{u}+\frac{K^{2}}{4\alpha^{2}} \right\rceil $ 
                \ENDIF
                \FOR{$k=1,\ldots,K$}
                    \STATE $\textrm{UCB}_{k}\gets \frac{1}{n_{k}}\sum^{n_{k}}_{n=1}Z_{k,n}+\sqrt{\frac{2\log(t-\tau)}{n_{k}}}$
                \ENDFOR
                \STATE $A_{t}\gets \argmax_{k\in\gK}\textrm{UCB}_{k}$
            \ENDIF
            \STATE Play arm $A_{t}$ and receive the reward $X_{A_{t},t}$.
            \STATE $n_{A_{t}}\gets n_{A_{t}}+1$;$Z_{A_{t},n_{A_{t}}}\gets X_{A_{t},t}$
            \IF{$n_{A_{t}}\geq w$}
                \IF{CD = True}
                    \STATE $k^{*}\gets \argmax_{k\in\gK}\sum^{n_{k}}_{\ell=1}Z_{k,\ell}/n_{k}$
                    \IF{$A_{t}=k^{*}$}
                        \IF{$\exists k\neq k^{*}$ such that $n_{k}>N_{I}$}
                            \IF{Skip($ n_{k}, n_{A_{t}} , Z_{k,1:n_{k}},Z_{A_{t},1:n_{A_{t}}}$)=False}
                                \STATE $\tau\gets t$, $u\gets 1$ and $n_{k} \gets 0 \forall k \in \mathcal{K}$
                            \ENDIF
                        \ENDIF
                    \ELSE
                        \IF{$n_{A_{t}}\geq N_{I}$}
                            \IF{Skip($n_{A_{t}},n_{k^{*}}, Z_{A_{t},1:n_{A_{t}}},  Z_{k^{*},1:n_{k^{*}}} $)=False}
                                \STATE $\tau\gets t$, $u\gets 1$ and $n_{k} \gets 0\ \forall k \in \mathcal{K}$
                            \ENDIF
                        \ENDIF
                    \ENDIF
                \ENDIF
            \ENDIF
        \ENDFOR
    \end{algorithmic}
\end{algorithm}

%% file: appendix/2-proof.tex
\section{Proof Detail}
\label{app:pf_detail}

This appendix provides the detailed proofs of the results presented in Section~\ref{sec:analysis} and Section~\ref{sec:extension}. Each proof is carefully elaborated to ensure clarity and rigor.

\subsection{Proof of Section~\ref{sec:analysis}}\label{app:pf_detail:not_extend}
In this subsection, we present the proofs of Theorem~\ref{thm:regret} in Section~\ref{sec:analysis}. In what follows, the first lemma bounds the regret accumulated during the diminishing exploration part of the algorithm. We denote by $R_{\mathrm{DE}}\left(\tau_{i-1}, \nu_{i}\right)$ as the regret caused by the exploration part of the algorithm and by $N_{\mathrm{DE},k}\left(\tau_{i-1}, \nu_{i}\right)$ the number of times that the arm $k$ is selected in the exploration phase from the previous alarm time to the next change point.

\begin{lemma}[Diminishing exploration regret]\label{lemma:de_regret}
    If the mean values of the arms remain the same during the time interval $[\tau_{i-1}, \nu_{i})$, then we have
    \begin{equation}
        N_{\mathrm{DE},k}\left(\tau_{i-1}, \nu_{i}\right)\leq \frac{2\alpha\sqrt{\nu_{i}-\tau_{i-1}}}{K}+\frac{3}{2},
    \end{equation}
    and
    \begin{equation}
        \E\left[R_{\mathrm{DE}}\left(\tau_{i-1}, \nu_{i}\right)\right]\leq 2\alpha\sqrt{\nu_{i}-\tau_{i-1}}+\frac{3}{2}K.
    \end{equation}
\end{lemma}

\begin{proof}
    Recall that $u_{i}^{(j)}$ is the beginning of the $j$-th uniform exploration session in the $i$-th segment. In Algorithm \ref{alg:main_alg}, the initial time of the first exploration session after each $\tau_{i}$ is given by:
    \begin{equation}
        u_{i}^{(1)}=\left\lceil\left(\alpha-\frac{K}{4\alpha}\right)^{2}\right\rceil, \label{eqn:algscheme_initial}
    \end{equation}
    and subsequent times follow the recursive equation:
    \begin{equation}
        u_{i}^{(j)}=\left\lceil u_{i}^{(j-1)}+\frac{K}{\alpha}\sqrt{u_{i}^{(j-1)}}+\frac{K^{2}}{4\alpha^{2}}\right\rceil\geq u_{i}^{(j-1)}+\frac{K}{\alpha}\sqrt{u_{i}^{(j-1)}}+\frac{K^{2}}{4\alpha^{2}}. \label{eqn:algscheme}
    \end{equation}

    Based on \eqref{eqn:algscheme_initial} and \eqref{eqn:algscheme}, one could easily check that the sequence $u_{i}^{(n)}$ satisfies that for every natural number $n$,
    \begin{equation}\label{eqn:u_i_bound}
        u_{i}^{(n)} \geq \left(\frac{\left(2n-3\right)K}{4\alpha}+\alpha\right)^{2}.
    \end{equation}

    Let $u_{i}^{(m)}$ be the last exploration start time in time interval $[\tau_{i-1}, \nu_{i})$. Then, we have
    \begin{equation}
        \E\left[R_{\mathrm{DE}}\left(\nu_{i}-\tau_{i-1}\right)\right]\leq mK. \label{eqn:R_DE}
    \end{equation}

    Additionally, we have:

    \begin{equation}
        \nu_{i}-\tau_{i-1}\geq u_{i}^{(m)} \geq \left(\frac{\left(2m-3\right)K}{4\alpha}+\alpha\right)^{2}\geq \left(\frac{2\E\left[R_{DE}\left(\nu_{i}-\tau_{i-1}\right)\right]-3K}{4\alpha}+\alpha\right)^{2}. \label{eqn:T_mK}
    \end{equation}


    Finally, based on the \eqref{eqn:R_DE} and \eqref{eqn:T_mK}, we can conclude that:              
    \begin{equation}
        \E\left[R_{\textrm{DE}}\left(\nu_{i}-\tau_{i-1}\right)\right]\leq 2\alpha\sqrt{\nu_{i}-\tau_{i-1}}-2\alpha^2+\frac{3}{2}K
        \leq 2\alpha\sqrt{\nu_{i}-\tau_{i-1}}+\frac{3}{2}K,
    \end{equation}
    and
    \begin{equation}
        N_{\textrm{DE},k}\left(\nu_{i}-\tau_{i-1}\right)\leq \frac{2\alpha\sqrt{\nu_{i}-\tau_{i-1}}}{K}+\frac{3}{2}.
    \end{equation}
\end{proof}

In the following lemma, we aim to explore how long it takes for a given arm to reach a certain number of samples through diminishing exploration.

\begin{lemma}[Samples-time steps transform]\label{lemma:samples-time}
    When each arm has accumulated \(n\) samples, the required time is as follows: 
    If the counting of the \(n\) samples begins immediately after a reset, the required time is given by
    \begin{equation}\label{eqn:Treset}
        T_{reset}\leq \left(\alpha+\frac{\left(2n-3\right)K}{4\alpha}+n\right)^2+K.
    \end{equation}
    However, if there is a delay of \(t_{d}\) time steps after the reset before the counting begins, the required time is given by 
    \begin{equation}\label{eqn:Ttd}
        T_{t_{d}}\leq 2n\left(\frac{K}{2\alpha}+1\right)\sqrt{t_{d}+1}+n^{2}\left(\frac{K}{2\alpha}+1\right)^{2}.
    \end{equation}
    
\end{lemma}

\begin{proof}

    We can derive the following from Equation~\ref{eqn:algscheme} in the proof of Lemma~\ref{lemma:de_regret}:
    
    \begin{equation}\label{eqn:u_iter}
    u^{(j)} \leq u^{(j-1)} + \frac{k}{\alpha} \sqrt{u^{(j-1)}} + \frac{k^2}{4\alpha^2} + 1 
    \leq \left( \sqrt{u^{(j-1)}} + \frac{k}{2\alpha} \right)^2 + 1 
    \leq \left( \sqrt{u^{(j-1)}} + \frac{k}{2\alpha} + 1 \right)^2.
    \end{equation}
    
    First, let us consider the case where the counting of \(n\) samples begins immediately after the reset. From Equation~\ref{eqn:algscheme_initial}, we can derive the following:
    
    \begin{equation}
    u^{(1)} \leq \left( \alpha - \frac{K}{4\alpha} + 1 \right)^2.
    \end{equation}
    
    Using Equation~\ref{eqn:u_iter}, we can further derive:
    
    \begin{equation}
    u^{(2)} \leq \left( \alpha - \frac{K}{4\alpha} + 1 + \frac{K}{2\alpha} + 1 \right)^2.
    \end{equation}
    
    Finally, we obtain:
    
    \begin{equation}
    u^{(n)} \leq \left( \alpha - \frac{K}{4\alpha} + 1 + (n-1)\left( \frac{K}{2\alpha} + 1 \right) \right)^2.
    \end{equation}
    
    Here, \(u^{(n)}\) represents the starting time of the \(n\)-th exploration block. The total time required to guarantee that all \(K\) arms have been sampled \(n\) times is therefore:
    
    \begin{equation}
    T_{\text{reset}} = u^{(n)} + K \leq \left( \alpha + \frac{(2n-3)K}{4\alpha} + n \right)^2 + K.
    \end{equation}
    
    Next, we consider how long it takes for each arm to be sampled \(n\) times after \(t_d\) time steps following the reset. We first assume that, prior to \(t_d\), each arm has already been sampled \(x\) times. For simplicity, we assume an ideal case where the exploration block starts exactly at time \(t_d + 1\). Therefore, we have:
    
    \begin{equation}
    u^{(x+1)} = t_d + 1.
    \end{equation}
    
    Since, in reality, the exploration block start time may not exactly coincide with \(t_d + 1\), we account for the possibility that it could begin at a later time by considering the next exploration block's start time. This allows us to bound the non-ideal case.
    
    Following the same approach as in the first part of the proof, we eventually obtain:
    
    \begin{equation}
    u^{(x+n+1)} \leq \left( \sqrt{t_d + 1} + n\left(\frac{K}{2\alpha} + 1\right) \right)^2.
    \end{equation}
    
    Finally, we derive that the total time required for each arm to be sampled \(n\) times after \(t_d\) time steps is:
    
    \begin{equation}
    T_{t_d} = u^{(x+n+1)} - u^{(x+1)} \leq 2n\left( \frac{K}{2\alpha} + 1 \right) \sqrt{t_d + 1} + n^2 \left( \frac{K}{2\alpha} + 1 \right)^2.
    \end{equation}

\end{proof}

Define $R\left(r,s\right) :=\sum^{s}_{t=r}\max_{k\in\gK}\E\left[X_{k,t}\right]-X_{A_{t},t}$ be the regret accumulated during $r$ and $s$. 
In the next lemma, we provide an upper bound on the regret accumulated from the $\left(i-1\right)$-th alarm time to the end of $\left(i-1\right)$-th segment, given that the previous change was successfully detected.

\begin{lemma}[Regret bound with stationary bandit]\label{lemma:regret_stat} 
Consider a stationary bandit interval with $\nu_{i-1}<\tau_{i-1}<\nu_{i}$. Condition on the successful detection events $\overline{F}_{i-1}$ and $D_{i-1}$, the expected regret accumulated during $\left(\tau_{i-1},\nu_{i}\right)$ can be bounded by
\begin{equation}
    \E\left[R\left(\tau_{i-1},\nu_{i}\right)\middle|\overline{F}_{i-1}D_{i-1} \right]\leq \tilde{C}+2\alpha\sqrt{s_{i}}+T\cdot \mathbb{P}\left(F_{i}\middle|\overline{F}_{i-1}D_{i-1}\right),
\end{equation}
where 
$\tilde{C}$ is as described in Theorem~\ref{thm:regret}. 
\end{lemma}

\begin{proof}
    For every $i$, we have
    \begin{subequations}
        \begin{align}
            \E\left[R\left(\tau_{i-1},\nu_{i}\right)\middle|\overline{F}_{i-1}D_{i-1} \right] &=\E\left[R\left(\tau_{i-1},\nu_{i}\right)\middle|F_{i}\overline{F}_{i-1}D_{i-1}\right]\mathbb{P}\left(F_{i}\middle|\overline{F}_{i-1}D_{i-1}\right)\\
            &\hspace{-20pt}+\E\left[R\left(\tau_{i-1},\nu_{i}\right)\middle|\overline{F}_{i}\overline{F}_{i-1}D_{i-1}\right]\mathbb{P}\left(\overline{F}_{i}\middle|\overline{F}_{i-1}D_{i-1}\right)\\
            &\hspace{-20pt}\leq T\cdot \mathbb{P}\left(F_{i}\middle|\overline{F}_{i-1}D_{i-1}\right)+\E\left[R\left(\tau_{i-1},\nu_{i}\right)\middle|\overline{F}_{i}\overline{F}_{i-1}D_{i-1}\right]. \label{eqn:regret_seg}
        \end{align}
    \end{subequations}
    Now, define $N_{k}\left(t_{1},t_{2}\right) :=\sum^{t_{2}}_{t=t_{1}}\1_{\left\{ A_{t}=k \right\}}$ to be the number of times that arm $k$ is selected by \Algref{alg:main_alg} from $t_{1}$ to $t_{2}$. Note that  
    \begin{equation}
        \E\left[R\left(\tau_{i-1},\nu_{i}\right)\middle|\overline{F}_{i}\overline{F}_{i-1}D_{i-1}\right]=\sum_{\Delta^{\left(i\right)}_{k}>0}\Delta^{\left(i\right)}_{k}\cdot \E\left[N_{k}\left(\tau_{i-1},\nu_{i}\right)\middle|\overline{F}_{i}\overline{F}_{i-1}D_{i-1}\right].
    \end{equation}
    To bound the second term of \eqref{eqn:regret_seg}, we further bound $N_{k}\left(\tau_{i-1},\nu_{i}\right)$ as follows,
    \begin{subequations}
        \begin{align}
            N_{k}\left(\tau_{i-1},\nu_{i}\right)=&\sum^{\nu_{i}}_{t=\tau_{i-1}+1}\1_{\left\{ A_{t}=k, \tau_{i}>\nu_{i}, N_{k}\left(\tau_{i-1},\nu_{i}\right)<l\right\}}+\sum^{\nu_{i}}_{t=\tau_{i-1}+1}\1_{\left\{ A_{t}=k, \tau_{i}>\nu_{i}, N_{k}\left(\tau_{i-1},\nu_{i}\right)\geq l\right\}}\\
            \leq &l+N_{DE,k}\left(\nu_{i}-\tau_{i-1}\right)+\sum^{\nu_{i}}_{t=\tau_{i-1}+1}\1_{\left\{ k=\argmax_{k\in\gK} \mathrm{UCB}_{k\in\gK}, \tau_{i}>\nu_{i}, N_{k}\left(\tau_{i-1},\nu_{i}\right)\geq l\right\}}\\
            \leq &l+\frac{2\alpha\sqrt{\nu_{i}-\tau_{i-1}}}{K}+\frac{3}{2}+\sum^{\nu_{i}}_{t=\tau_{i-1}+1}\1_{\left\{ k=\argmax_{k\in\gK} \mathrm{UCB}_{k\in\gK}, \tau_{i}>\nu_{i}, N_{k}\left(\tau_{i-1},\nu_{i}\right)\geq l\right\}}\label{eqn:N4}\\
            \leq &l+\frac{2\alpha\sqrt{s_{i}}}{K}+\frac{3}{2}+\sum^{\nu_{i}}_{t=\tau_{i-1}+1}\1_{\left\{ k=\argmax_{k\in\gK} \mathrm{UCB}_{k\in\gK}, \tau_{i}>\nu_{i}, N_{k}\left(\tau_{i-1},\nu_{i}\right)\geq l\right\}},
        \end{align}
    \end{subequations}
    \Eqref{eqn:N4} follows from Lemma \ref{lemma:de_regret}.
    Setting $l=\left\lceil 8\log T/\left(\Delta^{(i)}_{k}\right)^{2} \right\rceil $, and following the same steps as in the proof of Theorem 1 of \cite{auer2002finite}, we arrive at
    \begin{equation}
        \E\left[N_{k}\left(\tau_{i-1},\nu_{i}\right)\middle|\overline{F}_{i}\overline{F}_{i-1}D_{i-1}\right]\leq \frac{2\alpha\sqrt{s_{i}}}{K}+\frac{8\log T}{\left(\Delta^{(i)}_{k}\right)^{2}}+\frac{5}{2}+\frac{\pi^{2}}{3}+K.
    \end{equation}
    Putting everything together completes the proof.
\end{proof}

Theorem \ref{thm:regret} can then be proved by recursively applying Lemma~\ref{lemma:regret_stat}.
\begin{flushleft}
    {\bf Theorem~\ref{thm:regret} }
    The \Algref{alg:main_alg} can be combined with a CD algorithm, which achieves the expected regret upper bound as follows:
    \begin{multline} 
        \E\left[R\left(1,T\right)\right]\leq \sum^{M}_{i=1}\tilde{C}_{i}+2\alpha\sqrt{MT}
        +\sum^{M-1}_{i=1}\E\left[\tau_{i}-\nu_{i}\middle | D_{i}\overline{F}_{i}D_{i-1}\overline{F}_{i-1}\right]\\
        +T\sum^{M}_{i=1}\mathbb{P}\left(F_{i}\middle| \overline{F}_{i-1}D_{i-1}\right)+T\sum^{M-1}_{i=1}\mathbb{P}\left(\overline{D}_{i}\middle|\overline{F}_{i}\overline{F}_{i-1}D_{i-1}\right), \label{eqn:regret_bound_p}
    \end{multline}
    where $\tilde{C}_{i}=8\sum_{\Delta^{\left(i\right)}_{k}>0}\frac{\log T}{\Delta^{\left(i\right)}_{k}}+\left(\frac{5}{2}+\frac{\pi^{2}}{3}+K\right)\sum^{K}_{k=1}\Delta^{\left(i\right)}_{k}$.
\end{flushleft}    
\begin{proof}
    Recall that $R\left(r,s\right)=\sum^{s}_{t=r}\max_{k\in\gK}\E\left[X_{k,t}\right]-X_{A_{t},t}$, then  $\gR\left(T\right)=\E\left[R\left(1,T\right)\right]$. We have
    \begin{subequations}
        \label{10}
        \begin{align}
            \gR\left(T\right)&=\E\left[R\left(1,T\right)\right]\\
            &=\E\left[R\left(1,T\right)\middle|\overline{F}_{0}D_{0}\right] \label{eqn:R_init}\\
            &\leq \E\left[R\left(1,\nu_{1}\right)\middle|\overline{F}_{1}\overline{F}_{0}D_{0}\right]+\E\left[R\left(\nu_{1},T\right)\middle|\overline{F}_{1}\overline{F}_{0}D_{0}\right]+ T\cdot \mathbb{P}\left(F_{1}\middle| \overline{F}_{0}D_{0}\right)\label{eqn:R_conE1}\\
            &\leq  \tilde{C}_{1}+2\alpha\sqrt{\left(\nu_{1}-\nu_{0}\right)}+\E\left[R\left(\nu_{1},T\right)\middle| \overline{F}_{1}\overline{F}_{0}D_{0}\right]+T\cdot \mathbb{P}\left(F_{1}\middle| \overline{F}_{0}D_{0}\right),\label{eqn:R_conE2}
        \end{align}
    \end{subequations}
    where \eqref{eqn:R_init} holds because $\tau_{0}=0$, \eqref{eqn:R_conE1} is due to the law of total expectation and some trivial bounds, and
    \eqref{eqn:R_conE2} follows from Lemmas~\ref{lemma:regret_stat}. The third term in \eqref{eqn:R_conE2} is then further bounded as follows:

    \begin{subequations}
        \begin{align}
            \hspace{-20pt}\E\left[R\left(\nu_{1},T\right)\middle| \overline{F}_{1}\overline{F}_{0}D_{0}\right] 
            &\leq\E\left[R\left(\nu_{1},T\right)\middle| D_{1}\overline{F}_{1}\overline{F}_{0}D_{0}\right]+T\cdot \left(1-\mathbb{P}\left(D_{1}\middle| \overline{F}_{1}\overline{F}_{0}D_{0}\right)\right)\label{eqn:R_conF1}\\
            &\hspace{-60pt}\leq \E\left[R\left(\nu_{1},T\right)\middle| D_{1}\overline{F}_{1}\overline{F}_{0}D_{0}\right] + T\cdot\mathbb{P}\left(\overline{D}_{1}\middle| \overline{F}_{1}\overline{F}_{0}D_{0}\right)\label{eqn:R_conF2}\\     
            &\hspace{-60pt}= \E\left[R\left(\tau_{1},T\right)\middle| D_{1}\overline{F}_{1}\overline{F}_{0}D_{0}\right]+\E\left[R\left(\nu_{1},\tau_{1}\right)\middle| D_{1}\overline{F}_{1}\overline{F}_{0}D_{0}\right]+T\cdot\mathbb{P}\left(\overline{D}_{1}\middle| \overline{F}_{1}\overline{F}_{0}D_{0}\right)\label{eqn:reg_split1}\\
            &\hspace{-60pt}\leq\E\left[R\left(\tau_{1},T\right)\middle| \overline{F}_{1}D_{1}\right] + \E\left[\tau_{1}-\nu_{1}\middle| \overline{F}_{1}D_{1}\overline{F}_{0}D_{0}\right]\label{eqn:reg_split2}+T\cdot\mathbb{P}\left(\overline{D}_{1}\middle| \overline{F}_{1}\overline{F}_{0}D_{0}\right)\\
            &\hspace{-60pt}\leq\E\left[R\left(\tau_{1},T\right)\middle| \overline{F}_{1}D_{1}\right] + \E\left[\tau_{1}-\nu_{1}\middle| \overline{F}_{1}D_{1}\right]+T\cdot\mathbb{P}\left(\overline{D}_{1}\middle| \overline{F}_{1}\overline{F}_{0}D_{0}\right)\label{eqn:reg_split3},
        \end{align}
    \end{subequations}
    where \eqref{eqn:R_conF1} applies the law of total expectation and some trivial bounds.
    From here, we can set up the following recursion:

    \begin{subequations}
        \begin{align}
            &\E\left[R\left(1,T\right)\right] =\E\left[R\left(1,T\right)\middle|\overline{F}_{0}D_{0}\right] \\
            & \leq \E\left[R\left(\tau_{1},T\right)\middle|\overline{F}_{1}D_{1}\right]+\tilde{C}_{1}+2\alpha\sqrt{s_{1}-1}+\E\left[\tau_{1}-\nu_{1}\middle| \overline{F}_{1}D_{1}\right]\\
            &\hspace{+60pt}+T\cdot\mathbb{P}\left(F_{1}\middle| \overline{F}_{0}D_{0}\right)+T\cdot\mathbb{P}\left(\overline{D}_{1}\middle| \overline{F}_{1}\overline{F}_{0}D_{0}\right) \\
            & \leq \E\left[R\left(\tau_{2},T\right)\middle|\overline{F}_{2}D_{2}\right]+\sum^{2}_{i=1}\tilde{C}_{i}+2\alpha\sum^{2}_{i=1}\sqrt{s_{i}-1}\\ 
            &+\sum^{2}_{i=1}\E\left[\tau_{i}-\nu_{i}\middle| \overline{F}_{i-1}D_{i-1}\right]+T\sum^{2}_{i=1}\mathbb{P}\left(F_{i}\middle| \overline{F}_{i-1}D_{i-1}\right)+T\sum^{2}_{i=1}\mathbb{P}\left(\overline{D}_{i}\middle| \overline{F}_{i}\overline{F}_{i-1}D_{i-1}\right) \\
            &\hspace{+150pt}\vdots\nonumber\\
            & \leq \sum^{M}_{i=1}\tilde{C}_{i}+2\alpha\sum^{M}_{i=1}\sqrt{s_{i}}+\sum^{M-1}_{i=1}\E\left[\tau_{i}-\nu_{i}\middle| \overline{F}_{i-1}D_{i-1}\right]\\
            &\hspace{+60pt}+T\sum^{M}_{i=1}\mathbb{P}\left(F_{i}\middle| \overline{F}_{i-1}D_{i-1}\right)+T\sum^{M-1}_{i=1}\mathbb{P}\left(\overline{D}_{i}\middle| \overline{F}_{i}\overline{F}_{i-1}D_{i-1}\right)\\
            & \leq \sum^{M}_{i=1}\tilde{C}_{i}+2\alpha\sqrt{MT}+\sum^{M-1}_{i=1}\E\left[\tau_{i}-\nu_{i}\middle| \overline{F}_{i-1}D_{i-1}\right]\label{eqn:R_segment}\\
            &\hspace{+60pt}+T\sum^{M}_{i=1}\mathbb{P}\left(F_{i}\middle| \overline{F}_{i-1}D_{i-1}\right)+T\sum^{M-1}_{i=1}\mathbb{P}\left(\overline{D}_{i}\middle| \overline{F}_{i}\overline{F}_{i-1}D_{i-1}\right)
        \end{align}
    \end{subequations}
    where \eqref{eqn:R_segment} follows from the Cauchy–Schwarz inequality
    \begin{subequations}
        \begin{align}
            \left(\sum^{M}_{i=1}\sqrt{s_{i}}\right)^{2}\leq \left(\sum^{M}_{i=1}s_{i}\right)\left(\sum^{M}_{i=1}1\right)= M\sum^{M}_{i=1}s_{i}= MT,
        \end{align}
    \end{subequations}
\end{proof}

\subsubsection{Proof of Integration with Change Detectors of M-UCB}\label{app:pf_detail:not_extend:MUCB}

With the general regret bound in Theorem~\ref{thm:regret}, a regret bound of the M-UCB with the proposed diminishing exploration can be obtained by bounding $\mathbb{P}\left(F_{i}\middle|\overline{F}_{i-1}D_{i-1}\right)$, $\mathbb{P}\left(D_{i}\middle|\overline{F}_{i}\overline{F}_{i-1}D_{i-1}\right)$, and $\E\left[\tau_{i}-\nu_{i}\middle| \overline{F}_{i}D_{i}\overline{F}_{i-1}D_{i-1}\right]$.

First, in Lemma~\ref{lemma:prob_fa}, we show that the probability of false alarm is very small; thereby, its contribution to the regret is negligible. 

\begin{lemma}[Probability of false alarm]\label{lemma:prob_fa} 
Under \Algref{alg:main_alg} with parameter in \eqref{eqn:w_fix}, and \eqref{eqn:b_fix}, we have
\begin{equation}
    \mathbb{P}\left(F_{i}\middle|\overline{F}_{i-1}D_{i-1}\right)\leq wK\left(1-\left(1-\exp\left(-2b^{2}/w\right)\right)^{\left\lfloor T/w \right\rfloor }\right)\leq \frac{1}{T}.
\end{equation}
\end{lemma}

\begin{proof}
    Suppose that at time $t$, we have gathered $w$ samples of arm $k\in\gK$, namely $Y_{k,1}, Y_{k,2},\ldots, Y_{k,w}$, for change detection in line 17 of \Algref{alg:main_alg}, and we define
    \begin{equation}
        S_{k,t}= \sum^{w}_{\ell=w/2 +1}Y_{k,\ell}-\sum^{w/2}_{\ell=1}Y_{k,\ell}.            \label{eqn:s_def}
    \end{equation}
    Note that $S_{k,t}=0$ when there is insufficient (less than $w$) samples to trigger the change detection algorithm. By definition, we have
    \begin{equation}
        \tau_{k,i}=\inf \{t\geq \tau_{i-1}+w:\left\lvert S_{k,t}\right\rvert >b\}.
    \end{equation}
    Given that the events $D_{i-1}$ and $\bar{F}_{i-1}$ hold, we define $\tau_{k,i}$ as the first detection time of the $k$-th arm after $\nu_{i}$. Clearly, $\tau_{i}=\min _{k\in\gK} \left\{\tau_{k,i}\right\}$ as \Algref{alg:main_alg} would reset every time a change is detected. Using the union bound, we have
    \begin{subequations}
        \begin{align}
            \mathbb{P}\left(F_{i}\middle| \overline{F}_{i-1}D_{i-1}\right)=&\mathbb{P}\left(\max_{k\in\mathcal{K}} \sum_{t=\tau_{i-1}+1}^{\nu_{i}}\1_{\left\{A_{t}=k\right\}}\geq w,F_{i}\middle| \overline{F}_{i-1}D_{i-1}\right)\\
            &+\mathbb{P}\left(\max_{k\in\mathcal{K}} \sum_{t=\tau_{i-1}+1}^{\nu_{i}}\1_{\left\{A_{t}=k\right\}}<w,F_{i}\middle| \overline{F}_{i-1}D_{i-1}\right) \label{eqn:false_al_p_b}\\
            =&\mathbb{P}\left(F_{i}\middle| \overline{F}_{i-1}D_{i-1},\max_{k\in\mathcal{K}} \sum_{t=\tau_{i-1}+1}^{\nu_{i}}\1_{\left\{A_{t}=k\right\}}\geq w\right)\label{eqn:false_al_p_c}\\
            &\cdot \mathbb{P}\left(\max_{k\in\mathcal{K}} \sum_{t=\tau_{i-1}+1}^{\nu_{i}}\1_{\left\{A_{t}=k\right\}}\geq w\middle| \overline{F}_{i-1}D_{i-1}\right)\label{eqn:false_al_p_d}\\
            \leq &\mathbb{P}\left(F_{i}\middle|\overline{F}_{i-1}D_{i-1},\max_{k\in\mathcal{K}} \sum_{t=\tau_{i-1}+1}^{\nu_{i}}\1_{\left\{A_{t}=k\right\}}\geq w\right)\label{eqn:false_al_p_e}\\
            \leq & \sum^{K}_{k=1}\mathbb{P}\left(\tau_{k,i}\leq \nu_{i}\middle|\overline{F}_{i-1}D_{i-1}, \max_{k^{\prime}\in\mathcal{K}} \sum_{t=\tau_{i-1}+1}^{\nu_{i}}\1_{\left\{A_{t}=k^{\prime}\right\}}\geq w\right)\label{eqn:max_1}\\
            \leq & \sum^{K}_{k=1}\mathbb{P}\left(\tau_{k,i}\leq \nu_{i}\middle|\overline{F}_{i-1}D_{i-1}, \sum_{t=\tau_{i-1}+1}^{\nu_{i}}\1_{\left\{A_{t}=k\right\}}\geq w\right), \label{eqn:false_al_p_f}
        \end{align}
    \end{subequations}
    where the term in \eqref{eqn:false_al_p_b} is clearly equal to $0$ as there will be no false alarm if we do not even have sufficiently many observations to trigger the alarm as suggested by Algorithm~\ref{alg:CD_alg}. \Eqref{eqn:false_al_p_c} and \eqref{eqn:false_al_p_d} hold by the definition of conditional probability, \eqref{eqn:false_al_p_e} is due to the fact that the term in \eqref{eqn:false_al_p_d} is at most one, and \eqref{eqn:max_1} follows from the union bound. In \eqref{eqn:false_al_p_f}, if $k\neq k^{\prime}$, we cannot  guarantee that $\sum_{t=\tau_{i-1}+1}^{\nu_{i}}\1_{\left\{A_{t}=k^{\prime}\right\}}\geq w$. Hence, some $k$ might cause the probability in the \eqref{eqn:max_1} to be zeros. 

    For any $0\leq j\leq w-1$, define the stopping time
    \begin{equation}
        \tau_{k,i}^{(j)}:=\inf \{t=\tau_{i-1} + j+nw,n\in\mathbb{Z}^{+}:\left\lvert S_{k,t}\right\rvert >b\}.
    \end{equation}
    Clearly, $\tau_{k,i}=\min\{\tau_{k,i}^{(0)},\ldots,\tau_{k,i}^{(w-1)}\}$. Let us define, for any $0\leq j\leq w-1$,
    \begin{equation}
        \xi_{k,i}^{(j)}=\frac{\left(\tau_{k,i}^{(j)}-j-\tau_{i-1}\right)}{w}.
    \end{equation}
    Note that condition on the events $D_{i-1}$ and $\bar{F}_{i-1}$, $\xi_{k, i}^{(j)}$ is a geometric random variable with parameter $p := \mathbb{P }(\left\lvert S_{k,t}\right\rvert>b)$, because when fixing $j$, there is no overlap between the samples in the current window and the next.
    \begin{multline}
        \mathbb{P}\left(\tau_{k,i}^{(j)}=\tau_{i-1}+nw+j\middle|\overline{F}_{i-1}D_{i-1},\sum_{t=\tau_{i-1}+1}^{\nu_{i}}\1_{\left\{A_{t}=k\right\}}\geq w\right)\\=
            \mathbb{P}\left(\xi_{k,i} = n\middle|\overline{F}_{i-1}D_{i-1},\sum_{t=\tau_{i-1}+1}^{\nu_{i}}\1_{\left\{A_{t}=k\right\}}\geq w\right)
            =p(1-p)^{n-1}.
    \end{multline}

    Here, the inclusion of subsequent events as conditions should not impact the results, as when entering the change detection algorithm, those events have already occurred.
    Moreover, by union bound, we have that for any $k\in\gK$,
    \begin{subequations}
        \begin{align}
            \mathbb{P}\left(\tau_{k,i}\leq \nu_{i}\middle|\overline{F}_{i-1}D_{i-1},\sum_{t=\tau_{i-1}+1}^{\nu_{i}}\1_{\left\{A_{t}=k\right\}}\geq w\right)&\leq w\left( 1-\left(1-p\right)^{\left\lfloor \left(\nu_{i}-\tau_{i-1}\right)/w \right\rfloor}  \right)\\ 
            &\leq w\left( 1-(1-p)^{\left\lfloor T/w \right\rfloor}  \right).  \label{eqn:tau_ki}
        \end{align}
    \end{subequations}
    We further use the McDiarmid's inequality and the union bound to show that

    \begin{subequations}
        \begin{align}
            p&=\mathbb{P}\left(\left\lvert S_{k,t} \right\rvert>b\right)=\mathbb{P}\left(S_{k,t}>b\right)+\mathbb{P}\left(S_{k,t}<-b\right)\\
            &\leq 2\cdot \exp\left(-\frac{2b^{2}}{w}\right). \label{eqn:p}
        \end{align}
    \end{subequations}
    Using the result in \eqref{eqn:tau_ki} and \eqref{eqn:p} into \eqref{eqn:false_al_p_f}, 
    \begin{subequations}
        \begin{align}
            &\mathbb{P}\left(F_{i}\middle|\overline{F}_{i-1}D_{i-1}\right)\leq \sum^{K}_{k=1}w\left( 1-\left(1-2 \exp\left(-\frac{2b^{2}}{w}\right)\right)^{\left\lfloor T/w \right\rfloor}  \right) \\
            &= wK\left( 1-\left(1-2 \exp\left(-\frac{2b^{2}}{w}\right)\right)^{\left\lfloor T/w \right\rfloor}  \right).
        \end{align}
    \end{subequations}
    Moreover, applying $\left(1-x\right)^{a}>1-ax$ for any $a>1$ and $0<x<1$ and plugging the choice of $b = \sqrt{w\log\left(2KT^{2} \right)/2}$ as in \eqref{eqn:b_fix} shows the second inequality.
\end{proof}

Lemma~\ref{lemma:samples-time} ensures that, with high probability, the detection delay is confined within a tolerable interval. 

That is, each arm is sampled \(w/2\) times, and using equation~\ref{eqn:Ttd} from lemma~\ref{lemma:samples-time}, we select \(h_{i}\) as
\begin{equation}
h_{i} = \left\lceil w\left(\frac{K}{2\alpha}+1\right)\sqrt{s_{i}+1}+\frac{w^{2}}{4}\left(\frac{K}{2\alpha}+1\right)^{2} \right\rceil.
\end{equation}

\begin{lemma}[Probability of successful detection]\label{lemma:prob_delay} 
Consider a piecewise-stationary bandit environment. For any $\mu^{(i)},\mu^{(i+1)}\in\left[0,1\right]^{K}$ with parameters chosen in \eqref{eqn:w_fix} and \eqref{eqn:b_fix} 
and
\begin{equation}
h_{i} = \left\lceil w\left(\frac{K}{2\alpha}+1\right)\sqrt{s_{i}+1}+\frac{w^{2}}{4}\left(\frac{K}{2\alpha}+1\right)^{2} \right\rceil,
\end{equation}
for some $k\in\gK, i\geq 1$ and $c>0$, under the \Algref{alg:main_alg}, we have 
\begin{equation}
    \mathbb{P}\left(D_{i}\middle|\overline{F}_{i}\overline{F}_{i-1}D_{i-1}\right)\geq 1-\frac{1}{T}.
\end{equation}
\end{lemma}

\begin{proof}
    \begin{subequations}
        \begin{align}
            \mathbb{P}\left(D_{i}\middle|\overline{F}_{i}\overline{F}_{i-1}D_{i-1}\right)
            = &\mathbb{P}\left(\tau_{i}\leq \nu_{i}+h_{i}\middle|\overline{F}_{i}\overline{F}_{i-1}D_{i-1}\right)\\
            \geq &\max_{t \in \left\{ \nu_{i}+1,\ldots,\nu_{i}+h_{i}\right\}}\mathbb{P}\left(S_{\tilde{k},t}>b\middle| \overline{F}_{i}\overline{F}_{i-1}D_{i-1}\right)\label{eqn:alarm_pb3}\\
            \geq &\max_{j \in \left\{ 0,\ldots,w/2\right\}} \left(1-2\exp\left(-\frac{( j\left\lvert\delta^{(i)}_{\tilde{k}}\right\rvert-b)^{2}}{w}\right)\right)\label{eqn:alarm_pb4}\\
            = & 1-2\exp\left(-\frac{(w|\delta^{(i)}_{\tilde{k}}|/2-b)^{2}}{w}\right)\label{eqn:alarm_pb5}\\
            \geq & 1-2\exp\left(-\frac{wc^{2}}{4}\right)\label{eqn:alarm_pb6}.
        \end{align}
    \end{subequations}
    where $S_{\tilde{k},t}$ is defined in \eqref{eqn:s_def}, \eqref{eqn:alarm_pb4} follows from the McDiarmid’s inequality, and \eqref{eqn:alarm_pb5} is due to the fact that the maximum value is attained when $j=w/2$. Last, \eqref{eqn:alarm_pb6} is true for any choice of $w,b$ and $c$ such that $\delta^{(i)}_{\tilde{k}}\geq 2b/w+c$ holds. We thus set $w$ and $b$ as in \eqref{eqn:w_fix} and \eqref{eqn:b_fix}, respectively, and choose $c=2\sqrt{\log\left(2T\right)/w}$, which leads to $\mathbb{P}\left(D_{i}\middle|\overline{F}_{i}\overline{F}_{i-1}D_{i-1}\right)\geq 1-1/T$.
\end{proof}
    
Lemma~\ref{lemma:exp_delay} further bounds the expected detection delay in the situation where the change detection algorithm successfully detects the change within the desired interval.

\begin{lemma}[Expected detection delay]\label{lemma:exp_delay} 
Consider a piecewise-stationary bandit environment. For any $\mu^{(i)},\mu^{(i+1)}\in\left[0,1\right]^{K}$ with parameters chosen in \eqref{eqn:w_fix} and \eqref{eqn:b_fix} 
and
\begin{equation}
h_{i} = \left\lceil w\left(\frac{K}{2\alpha}+1\right)\sqrt{s_{i}+1}+\frac{w^{2}}{4}\left(\frac{K}{2\alpha}+1\right)^{2} \right\rceil,
\end{equation}
for some $K\in\gK, i\geq 1$ and $c>0$, under the \Algref{alg:main_alg}, we have
\begin{equation}
    \E\left[\tau_{i}-\nu_{i}\middle| \overline{F}_{i}D_{i}\overline{F}_{i-1}D_{i-1}\right]\leq h_{i}.
\end{equation}
\end{lemma}

\begin{proof}
    For any $1\leq i\leq M$, we have
    \begin{subequations}
        \begin{align}
            \E\left[\tau_{i}-\nu_{i}\middle| \overline{F}_{i}D_{i}G_{i}\overline{F}_{i-1}D_{i-1}\right]  
            &=\sum^{h_{i}}_{j=1}\mathbb{P}\left(\tau_{i}\geq\nu_{i}+j\middle| \overline{F}_{i}D_{i}G_{i}\overline{F}_{i-1}D_{i-1}\right)
            \leq h_{i}.
        \end{align}
    \end{subequations}
\end{proof}

Plugging the bounds in Lemmas ~\ref{lemma:prob_fa},~\ref{lemma:prob_delay} and~\ref{lemma:exp_delay} into Theorem~\ref{thm:regret} shows the following regret bound in Corollary~\ref{cor:regret_MUCB}.

\begin{flushleft}
    {\bf Corollary~\ref{cor:regret_MUCB} }
    Combining Algorithm~\ref{alg:main_alg} and~\ref{alg:CD_alg} with the parameters in Equation~\ref{eqn:w_fix}, and Equation~\ref{eqn:b_fix} achieves the expected regret upper bound as follows:
    \begin{multline} 
     \E\left[R\left(1,T\right)\right]\leq \underbrace{\sum^{M}_{i=1}\tilde{C}_{i}}_{(a)}+\underbrace{2\alpha\sqrt{MT}}_{(b)}
    +\underbrace{w\left(\frac{K}{2\alpha}+1\right)\sqrt{M\left(T+M\right)}}_{(c)}\\
    +\underbrace{\frac{w^{2}M}{4}\left(\frac{K}{2\alpha}+1\right)^{2}}_{(c)}+\underbrace{2M}_{(d)},
    \end{multline}\label{eqn:regret_bound_MUCB_p}
    where $\tilde{C}_{i}=8\sum_{\Delta^{\left(i\right)}_{k}>0}\frac{\log T}{\Delta^{\left(i\right)}_{k}}+\left(\frac{5}{2}+\frac{\pi^{2}}{3}+K\right)\sum^{K}_{k=1}\Delta^{\left(i\right)}_{k}$. By setting $\alpha = c\sqrt{K\log{\left(KT\right)}}$ for some constant $c$, the expected regret is upper-bounded by $\mathcal{O}(\sqrt{KMT\log{T}})$.
\end{flushleft}

\subsubsection{Proof of Integration with Change Detectors of GLR-UCB}\label{app:pf_detail:not_extend:GLRUCB}
First, we introduce the function $\mathcal{J}$, originally introduced by \cite{JMLR:v22:18-798}, 
\begin{equation}
    \mathcal{J}(x):=2\tilde{g}\left(\frac{g^{-1}\left(1+x\right)+\ln\left(\pi^{2}/3\right)}{2}\right),
\end{equation}
where $g^{-1}\left(y\right)$ is the inverse function of $g\left(y\right):=y-\ln\left(y\right)$ defined for $y\geq 1$, and for any $x\geq 0$, $\tilde{g}\left(x\right) := e^{1/g^{-1}\left(x\right)}g^{-1}\left(x\right)$ if $x\geq g^{-1}\left(1/\ln\left(3/2\right)\right)$ and $\tilde{g}\left(x\right)=\left(3/2\right)\left(x-\ln\left(\ln\left(3/2\right)\right)\right)$ otherwise. 
We select the threshold function
\begin{equation}
\beta\left(n,\epsilon\right):=2\mathcal{J}\left(\frac{\log{(3n\sqrt{n}/\epsilon)}}{2}\right)+6\log{(1+\log{n})}, 
\end{equation}
and define $h_i$ in successful detection events $D_i$ to be $h_{i}:=h_{i}\left(\alpha,\epsilon\right)$ with $h_{0}\left(\alpha,\epsilon\right):=0$ and for $i>0$,
\begin{align}
    h_{i}\left(\alpha,\epsilon\right) :=\left\lceil 2\left(\frac{4}{\left(\delta^{\left(i\right)}\right)^{2}}\beta\left(T,\epsilon\right)+2\right)\left(\frac{K}{2\alpha}+1\right)\sqrt{s_{i}+1}\right. \nonumber\\ \left.+\left(\frac{4}{\left(\delta^{\left(i\right)}\right)^{2}}\beta\left(T,\epsilon\right)+2\right)^{2}\left(\frac{K}{2\alpha}+1\right)^{2} \right\rceil, \label{eqn:glr_dimi_delay}
\end{align}
which guarantees that with the proposed diminishing exploration, we will observe 
\begin{equation}
    \left\lceil\frac{4}{\left(\delta^{\left(i\right)}\right)^{2}}\beta\left(\frac{3}{2}s_{i},\epsilon\right)+1\right\rceil,
\end{equation}
post-change samples for each $k\in\gK$ after $\nu_i$. We analyze the GLR-UCB with diminishing exploration under the following assumption:
\begin{assumption}\label{ass:glr_delay}
    $\nu_{i}-\nu_{i-1}\geq 2\max\left\{h_{i}, h_{i-1}\right\}$ for all $i\in\left\{1,\ldots,M\right\}$.
\end{assumption}

Following the proof of Lemma 8 in \cite{besson2022efficient}, we can show the following lemma:
\begin{lemma}\label{lemma:glr_fa_md}
    Under assumption~\ref{ass:glr_delay} and Equation~\ref{eqn:glr_dimi_delay}, it holds that
    \begin{equation}
        \sum^{M}_{i=1}\mathbb{P}\left(F_{i}\middle| \overline{F}_{i-1}D_{i-1}\right)+\sum^{M-1}_{i=1}\mathbb{P}\left(\overline{D}_{i}\middle|\overline{F}_{i}\overline{F}_{i-1}D_{i-1}\right)\leq \epsilon\left(K+1\right)M.
    \end{equation}
\end{lemma}

Plugging \eqref{eqn:glr_dimi_delay} and Lemma~\ref{lemma:glr_fa_md} into Theorem~\ref{thm:regret} shows the following Corollary~\ref{cor:regret_glrUCB}.

{\bf Corollary~\ref{cor:regret_glrUCB}}
Combining Algorithm~\ref{alg:main_alg} and ~\ref{alg:glrCD_alg} with $\beta$ function in \eqref{eqn:beta} achieves the expected regret upper bound as follows:
\begin{multline} 
    \E\left[R\left(1,T\right)\right]\leq \underbrace{\sum^{M}_{i=1}\tilde{C}_{i}}_{(a)}+\underbrace{2\alpha\sqrt{MT}}_{(b)}
    +\underbrace{\left(\frac{4}{\left(\delta^{\left(i\right)}\right)^{2}}\beta\left(T,\epsilon\right)+2\right)^{2}\left(\frac{K}{2\alpha}+1\right)^{2}M}_{(c)}\\
    +\underbrace{2\left(\frac{4}{\left(\delta^{\left(i\right)}\right)^{2}}\beta\left(T,\epsilon\right)+2\right)\left(\frac{K}{2\alpha}+1\right)\sqrt{M\left(T+M\right)}}_{(c)}+\underbrace{\epsilon \left(K+1\right)M}_{(d)},
\end{multline}\label{eqn:regret_bound_glrUCB_p}
where $\tilde{C}_{i}=8\sum_{\Delta^{\left(i\right)}_{k}>0}\frac{\log T}{\Delta^{\left(i\right)}_{k}}+\left(\frac{5}{2}+\frac{\pi^{2}}{3}+K\right)\sum^{K}_{k=1}\Delta^{\left(i\right)}_{k}$. By setting $\alpha = c\sqrt{K\log{\left(KT\right)}}$ for some constant $c$ and $\epsilon=1/\sqrt{T}$, the expected regret is upper-bounded by $\mathcal{O}(\sqrt{KMT\log{T}})$.

\subsection{Proof of Section~\ref{sec:extension}}\label{app:pf_detail:extend}

\begin{figure}[b]
    \centering
    \includegraphics[width = 0.8\textwidth]{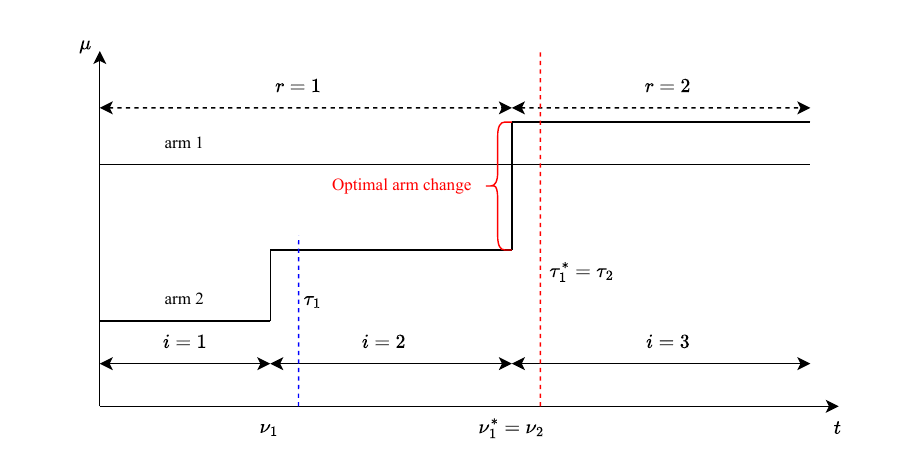}
    \vspace{-15pt}
    \caption{This figure allows us to easily compare the notation in Section~\ref{sec:analysis} and~\ref{sec:extension}.}
    \label{fig:notation}
\end{figure}

This subsection covers the proofs of the extended results discussed in Section~\ref{sec:extension}. These extensions include advanced integrations and new theoretical insights.

\begin{corollary}
    We can extend Lemma \ref{lemma:de_regret} to the case where we only care about the optimal arm changing to another one. Then, the number of times that arm $k$ is selected in the exploration phase and the diminishing exploration regret during the time interval $[\tau_{r-1}^{*}, \nu_{r}^{*})$ would be
    \begin{equation}
        N_{\mathrm{DE},k}\left(\tau_{r-1}^{*}, \nu_{r}^{*}\right)\leq \frac{2\alpha\sqrt{\nu_{r}^{*}-\tau_{r-1}^{*}}}{K}+\frac{3}{2},
    \end{equation}
    and
    \begin{equation}
        \E\left[R_{\mathrm{DE}}\left(\tau_{r-1}^{*}, \nu_{r}^{*}\right)\right]\leq 2\alpha\sqrt{ \nu_{r}^{*}-\tau_{r-1}^{*}}+\frac{3}{2}K.
    \end{equation}    
\end{corollary}

\begin{corollary}
    We can extend Lemma \ref{lemma:regret_stat} to the case where we only care about the optimal arm changing to another one. Then, the number of times that arm $k$ is selected in the exploration phase and the diminishing exploration regret during $[\tau_{r-1}^{*}, \nu_{r}^{*})$ can be bounding by 
    \begin{equation}
        \E\left[R\left(\tau_{r-1}^{*},\nu_{r}^{*}\right)\middle|\overline{F}_{r-1}^{*}D_{r-1}^{*} \right]\leq \tilde{C}+2\alpha\sqrt{s_{r}^{*}}+T\cdot \mathbb{P}\left(F_{r}^{*}\middle|\overline{F}_{r-1}^{*}D_{r-1}^{*}\right),
    \end{equation}
    
\end{corollary}

\begin{lemma}\label{lemma:opt_fal}
    The false alarm of the super segment can be bounded as follow:
    \begin{equation}
        \mathbb{P}\left(F_{r}^{*}\middle| \overline{F}_{r-1}^{*}D_{r-1}^{*}\right)\leq \sum^{K}_{k=1}\mathbb{P}\left(\overline{Ignore}, \textrm{ optimal arm no change }\middle| \textrm{arm $k$ alarm, } \overline{F}_{r-1}^{*}D_{r-1}^{*}\right)
    \end{equation}
\end{lemma}
\begin{proof}
    Conditioning on $\overline{F}_{r-1}^{*}D_{r-1}^{*}$ holds, we have the event 
    \begin{equation}
        F_{r}^{*}=\bigcup_{k=1}^{K}\left\{\overline{Ignore}, \textrm{optimal arm no change, arm $k$ alarm}\right\}.
    \end{equation}
    Using the union bound, 
    \begin{subequations}
        \begin{align}
            \mathbb{P}\left(F_{r}^{*}\middle| \overline{F}_{r-1}^{*}D_{r-1}^{*}\right)&=\mathbb{P}\left(\bigcup_{k=1}^{K}\left\{\overline{Ignore}, \textrm{optimal arm no change, arm $k$ alarm}\right\}\middle| \overline{F}_{r-1}^{*}D_{r-1}^{*}\right)\\
            &\leq \sum^{K}_{k=1}\mathbb{P}\left(\overline{Ignore}, \textrm{optimal arm no change, arm $k$ alarm}\middle| \overline{F}_{r-1}^{*}D_{r-1}^{*}\right)\\
            &= \sum^{K}_{k=1}\mathbb{P}\left(\overline{Ignore}, \textrm{optimal arm no change}\middle| \textrm{arm $k$ alarm, } \overline{F}_{r-1}^{*}D_{r-1}^{*}\right)\\
            &\times\mathbb{P}\left(\textrm{arm $k$ alarm}\middle| \overline{F}_{r-1}^{*}D_{r-1}^{*}\right)\\
            &\leq \sum^{K}_{k=1}\mathbb{P}\left(\overline{Ignore}, \textrm{optimal arm no change}\middle| \textrm{arm $k$ alarm, } \overline{F}_{r-1}^{*}D_{r-1}^{*}\right)
        \end{align}
    \end{subequations}
\end{proof}

\begin{lemma}\label{lemma:opt_md}
    The miss detection probability of the super segment can be bounded as follow:
    \begin{equation}
        \mathbb{P}\left(\overline{D}_{r}^{*}\middle| \overline{F}_{r}^{*}\overline{F}_{r-1}^{*}D_{r-1}^{*}\right)\leq \sum^{K}_{k=1}\mathbb{P}\left(\overline{D}_{r,k}\middle| \overline{F}_{r}^{*}\overline{F}_{r-1}^{*}D_{r-1}^{*}\right)\label{eqn:D_star_bar}
    \end{equation}
\end{lemma}

\begin{proof}
    We make some modifications to event $D_{i}$ in Section~\ref{sec:analysis} and extend $D_{r,k}$ as the arm $k$ alarm in the tolerate delay (regardless of whether to ignore or not) after the $r$-th optimal arm change.
    Conditioning on $\overline{F}_{r}^{*}\overline{F}_{r-1}^{*}D_{r-1}^{*}$ holds, we have the event 
    \begin{equation}
        \overline{D}_{r}^{*}=\bigcup_{k=1}^{K}\overline{D}_{r,k}.
    \end{equation}
    Using the union bound, we can get the result as Equation~\ref{eqn:D_star_bar}.
\end{proof}

\begin{theorem}[General form of regret bound]\label{thm:regret_bound_extend}
    Insert Algorithm~\ref{alg:skip} into Algorithm~\ref{alg:main_alg} and combine with a CD algorithm, which achieves the expected regret upper bound as follows:
    \vspace{-10pt}
    \begin{multline} 
        \E\left[R\left(1,T\right)\right]\leq \underbrace{\sum^{S}_{r=1}\tilde{C}_{r}^{*}}_{(a)}+\underbrace{2\alpha\sqrt{ST}}_{(b)}
        +\underbrace{\sum^{S-1}_{r=1}\left(\E\left[\tau_{r}^{*}-\nu_{r}^{*}\middle | D_{i}^{*}\overline{F}_{i}^{*}D_{i-1}^{*}\overline{F}_{i-1}^{*}\right]+d_{I,r}\right)}_{(c)}\\
        +\underbrace{T\sum^{S-1}_{r=1}\mathbb{P}\left(F_{r}^{*}\middle| \overline{F}_{r-1}^{*}D_{r-1}^{*}\right)+\mathbb{P}\left(\overline{D}_{r}^{*}\middle|\overline{F}_{r}^{*}\overline{F}_{r-1}^{*}D_{r-1}^{*}\right)}_{(d)}, \label{eqn:regret_bound_extend}
    \end{multline}
    where $\tilde{C}_{r}^{*}=8\sum_{\min_{\nu_{r-1}^{*}\leq t\leq \nu_{r}^{*}}\Delta_{k,t}>0}\frac{\log T}{\min_{\nu_{r-1}^{*}\leq t\leq \nu_{r}^{*}}\Delta_{k,t}}+\left(\frac{5}{2}+\frac{\pi^{2}}{3}+K\right)\sum^{K}_{k=1}\max_{\nu_{r-1}^{*}\leq t\leq \nu_{r}^{*}}\Delta_{k,t}$, and {$d_{I,r}$ is the upper bound of expected duration required for sufficient samples to decide whether to ignore under the effect of diminishing exploration, and this variable varies depending on the CD algorithm.} Moreover, we can transform Equation~\ref{eqn:regret_bound_extend} into another form using lemma~\ref{lemma:opt_fal} and~\ref{lemma:opt_md} as follows:
    \begin{multline} 
        \E\left[R\left(1,T\right)\right]\leq \underbrace{\sum^{S}_{r=1}\tilde{C}_{r}^{*}}_{(a)}+\underbrace{2\alpha\sqrt{ST}}_{(b)}
        +\underbrace{\sum^{S-1}_{r=1}\left(\E\left[\tau_{r}^{*}-\nu_{r}^{*}\middle | D_{i}^{*}\overline{F}_{i}^{*}D_{i-1}^{*}\overline{F}_{i-1}^{*}\right]+d_{I,r}\right)}_{(c)}\\
        +\underbrace{T\sum^{S-1}_{r=1}\sum^{K}_{k=1}\mathbb{P}\left(\overline{Ignore},\textrm{optimal arm no change}\middle| \textrm{arm $k$ alarm, }\overline{F}_{r-1}^{*}D_{r-1}^{*}\right)+T\sum^{S-1}_{r=1}\sum^{K}_{k=1}\mathbb{P}\left(\overline{D_{r,k}}\middle|\overline{F}_{r}^{*}\overline{F}_{r-1}^{*}D_{r-1}^{*}\right)}_{(d)}, \label{eqn:regret_bound_extend2}
    \end{multline}
\end{theorem}

\begin{lemma}\label{lemma:ignore_suff_samples}
    Suppose an arm $k\in\gK$ changes at time $\nu$ and raises an alarm at time $\tau$, but the optimal arm is the same one, and we choose a variable $N_{I}$ such that $N_{k}\left(\nu,\tau\right)+N_{I}\geq \frac{4\xi \log T}{\Delta_{\min}^{2}}$ and $\xi =1$, we have
    \begin{equation}
        \mathbb{P}\left(\overline{Ignore}, \textrm{optimal arm no change}\middle| \textrm{arm $k$ alarm }, \overline{F}_{r-1}^{*}D_{r-1}^{*}\right)\leq \frac{2}{T^{2}}
    \end{equation}
\end{lemma}

\begin{proof}
    Condition on arm $k$ alarms, $\overline{F}_{r-1}^{*}$ and $D_{r}^{*}$ hold, the event
    \begin{subequations}
        \begin{align}
            \left\{\overline{Ignore}, \textrm{optimal arm no change}\right\}&\subset \left\{\hat{\mu}_{k}\geq \mu_{k}+\sqrt{\frac{\xi\log{T}}{N_{k}\left(\nu,\tau\right)+N_{I}}}\right\}\label{eqn:over_est}\\
            &\cup \left\{\hat{\mu}_{k^{*}}\leq \mu_{k^{*}}-\sqrt{\frac{\xi\log{T}}{N_{k^{*}}\left(\nu,\tau\right)+N_{I}}}\right\}\label{eqn:under_est}\\
            &\cup \left\{\mu_{k^{*}}-\mu_{k}<2\sqrt{\frac{\xi\log{T}}{N_{k}\left(\nu,\tau\right)+N_{I}}}, N_{k}\left(\nu,\tau\right)+N_{I}\geq \frac{4\xi \log T}{\Delta_{\min}^{2}}\right\}\label{eqn:vanish_event}.
        \end{align}
    \end{subequations}
    The third event will vanish because
    \begin{equation}
        2\sqrt{\frac{\xi\log{T}}{N_{k}\left(\nu,\tau\right)+N_{I}}}\leq 2\sqrt{\frac{\xi\log{T}\Delta_{\min}^{2}}{4\xi\log{T}}}=\Delta_{\min}\leq \mu_{k^{*}}-\mu_{k}\label{eqn:contradition}
    \end{equation}
    Equation~\ref{eqn:contradition} substitutes the latter term of event~\ref{eqn:vanish_event} into the former term. As a result of the substitution, it is determined that this event cannot occur, hence the probability is zero. Therefore, we only need to consider events~\ref{eqn:over_est} and~\ref{eqn:under_est}. Using the Chernoff-Hoeffding bound, we can obtain
    \begin{equation}
        \mathbb{P}\left(\hat{\mu}_{k}\geq \mu_{k}+\sqrt{\frac{\xi\log{T}}{N_{k}\left(\nu,\tau\right)+N_{I}}}\right)\leq T^{-2\xi}.
    \end{equation}
    \begin{equation}
        \mathbb{P}\left(\hat{\mu}_{k^{*}}\leq \mu_{k^{*}}-\sqrt{\frac{\xi\log{T}}{N_{k^{*}}\left(\nu,\tau\right)+N_{I}}}\right) \leq T^{-2\xi}
    \end{equation}
    If we choose $\xi = 1$, then
    \begin{subequations}
        \begin{align}
            \mathbb{P}\left(\hat{\mu}_{k}\geq \mu_{k}+\sqrt{\frac{\xi\log{T}}{N_{k}\left(\nu,\tau\right)+N_{I}}}\right)+\mathbb{P}\left(\hat{\mu}_{k^{*}}\leq \mu_{k^{*}}-\sqrt{\frac{\xi\log{T}}{N_{k^{*}}\left(\nu,\tau\right)+N_{I}}}\right)\leq \frac{2}{T^{2}}
        \end{align}
    \end{subequations}
\end{proof}

\subsubsection{Proof of Integration with Change Detectors of M-UCB}\label{app:pf_detail:extend:MUCB}
\begin{assumption}\label{ass:minimum_gap_extend}
    The algorithm knows a lower bound $\delta>0$ such that $\delta\leq\min_{i}\max_{k\in\mathcal{K}}\delta^{\left(i\right)}_{k}$.
\end{assumption}

\begin{assumption}\label{asm:seg_length_extend}
    $s^{*}_{r}= \Omega\left(\left(\log{KT}+\sqrt{K\log{KT}}\right)\sqrt{s^{*}_{r-1}}\right)$.
\end{assumption}
In particular, if $s_r^{*} = \Theta\left(\left(\log{KT}+\sqrt{K\log{KT}}\right)^{2(1+\epsilon)}\right)$ for every $i$, Assumption~\ref{asm:seg_length_extend} holds.

\begin{flushleft}
    {\bf Corollary~\ref{cor:regret_bound_extend_MUCB}} Combining Algorithm~\ref{alg:main_alg} and~\ref{alg:CD_alg} with the parameters in Equation~\ref{eqn:w_fix2}, and Equation~\ref{eqn:b_fix} achieves the expected regret upper bound as follows:
    \begin{multline} 
     \E\left[R\left(1,T\right)\right]\leq \underbrace{\sum^{S}_{r=1}\tilde{C}_{r}^{*}}_{(a)}+\underbrace{2\alpha\sqrt{ST}}_{(b)}
    +\underbrace{w\left(\frac{K}{2\alpha}+1\right)\sqrt{S\left(T+S\right)}}_{(c)}\\
    +\underbrace{\frac{w^{2}S}{4}\left(\frac{K}{2\alpha}+1\right)^{2}}_{(c)}+\underbrace{2S}_{(d)},
    \end{multline}\label{eqn:regret_bound_extend_MUCB_p}
    where $\tilde{C}_{r}^{*}=8\sum_{\min_{\nu_{r-1}^{*}\leq t\leq \nu_{r}^{*}}\Delta_{k,t}>0}\frac{\log T}{\min_{\nu_{r-1}^{*}\leq t\leq \nu_{r}^{*}}\Delta_{k,t}}+\left(\frac{5}{2}+\frac{\pi^{2}}{3}+K\right)\sum^{K}_{k=1}\max_{\nu_{r-1}^{*}\leq t\leq \nu_{r}^{*}}\Delta_{k,t}$. By setting $\alpha = c\sqrt{K\log{\left(KT\right)}}$ for some constant $c$, the expected regret is upper-bounded by $\mathcal{O}(\sqrt{KST\log{T}})$.
\end{flushleft}

\begin{proof}
    We can substitute lemma~\ref{lemma:exp_delay}, lemma~\ref{lemma:prob_fa}, and lemma~\ref{lemma:prob_delay} into terms (c) and (d) of Equation~\ref{eqn:regret_bound_extend2} respectively, and $d_{I,r}$ could be zero because $w$ (window size) samples are sufficient to determine whether to ignore. We don't need to make an additional decision interval to ensure that the number of samples is sufficient for ignoring. 
\end{proof}

\subsubsection{Proof of Integration with Change Detectors of GLR-UCB}\label{app:pf_detail:extend:GLRUCB}

\begin{flushleft}
    {\bf Corollary~\ref{cor:regret_bound_extend_glrUCB}} Combining Algorithm ~\ref{alg:main_alg} and ~\ref{alg:glrCD_alg} with $\beta$ function in Equation~\ref{eqn:beta} achieves the expected regret upper bound as follows:
    \begin{multline} 
     \E\left[R\left(1,T\right)\right]\leq \underbrace{\sum^{S}_{r=1}\tilde{C}_{r}^{*}}_{(a)}+\underbrace{2\alpha\sqrt{ST}}_{(b)}
    +\underbrace{\left[\left(\frac{4}{\delta^{2}}\beta\left(T,\epsilon\right)+2\right)^{2}+\frac{16\log^{2}{T}}{\Delta_{\min}^{4}}\right]\left(\frac{K}{2\alpha}+1\right)^{2}M}_{(c)}\\
    +\underbrace{4\left[\left(\frac{2}{\delta^{2}}\beta\left(T,\epsilon\right)+1\right)+\frac{8\log{T}}{\Delta_{\min}^{2}}\right]\left(\frac{K}{2\alpha}+1\right)\sqrt{M\left(T+M\right)}}_{(c)}+\underbrace{\frac{2KS}{T}+\epsilon SKT}_{(d)},
    \end{multline}\label{eqn:regret_bound_extend_glrUCB_p}
    where $\tilde{C}_{r}^{*}=8\sum_{\min_{\nu_{r-1}^{*}\leq t\leq \nu_{r}^{*}}\Delta_{k,t}>0}\frac{\log T}{\min_{\nu_{r-1}^{*}\leq t\leq \nu_{r}^{*}}\Delta_{k,t}}+\left(\frac{5}{2}+\frac{\pi^{2}}{3}+K\right)\sum^{K}_{k=1}\max_{\nu_{r-1}^{*}\leq t\leq \nu_{r}^{*}}\Delta_{k,t}$. By setting $\alpha = c\sqrt{K\log{\left(KT\right)}}$ for some constant $c$ and $\epsilon=1/\sqrt{T}$, the expected regret is upper-bounded by $\mathcal{O}(\sqrt{KST\log{T}})$.
\end{flushleft}

\begin{proof}
    We can substitute Equation~\ref{eqn:glr_dimi_delay}, and lemma~\ref{lemma:glr_fa_md} into terms (c) and (d) of Equation~\ref{eqn:regret_bound_extend2} respectively. 
    $d_{I,r}$ is the upper bound of the expected duration required for $N_{I}$ samples under the effect of decreasing exploration, and $\tau_{r}$ is the time the CD algorithm alarms after $\tau^{*}_{r-1}$ but not through the skipping mechanism. Here, we choose $N_{I}=\frac{4\log{T}}{\Delta_{\min}^{2}}$. Using lemma~\ref{lemma:samples-time}, we can obtain:
    \begin{subequations}
        \begin{align}
            d_{I,r}&\leq \frac{8\log{T}}{\Delta_{\min}^{2}}\left(\frac{K}{2\alpha}+1\right)\sqrt{\tau_{r}-\tau_{r-1}^{*}+1}+\frac{16\log^{2}{T}}{\Delta_{\min}^{4}}\left(\frac{K}{2\alpha}+1\right)^{2}\\
            &\leq \frac{8\log{T}}{\Delta_{\min}^{2}}\left(\frac{K}{2\alpha}+1\right)\sqrt{\tau_{r}-\nu_{r-1}^{*}+1}+\frac{16\log^{2}{T}}{\Delta_{\min}^{4}}\left(\frac{K}{2\alpha}+1\right)^{2}\\
            &\leq \frac{8\log{T}}{\Delta_{\min}^{2}}\left(\frac{K}{2\alpha}+1\right)\left(\sqrt{s_{r}^{*}+1}+\sqrt{s_{r+1}^{*}+1}\right)+\frac{16\log^{2}{T}}{\Delta_{\min}^{4}}\left(\frac{K}{2\alpha}+1\right)^{2}
        \end{align}
    \end{subequations}
    
\end{proof}

%% file: appendix/3-parameter.tex
\section{Algorithms and Parameters Tuning}\label{app:para}
In this appendix, we provide an explanation of our parameter selection. For M-UCB, the window size $w$ is set to 200 unless otherwise specified; however, for the last data point ($M=100$) in Figure~\ref{fig:M}, we chose $w=50$ due to the limitations inherent to change detection. We compute the change detection threshold $b_{\textrm{M-UCB}} = \sqrt{w/2 \log\left(2KT^{2}\right)}$ following the original formulation in \cite{cao2019nearly}. Additionally, the uniform exploration rate $\gamma_{\textrm{M-UCB}} = \sqrt{MK \log{T}/T}$ is determined as initially stated in \cite{besson2022efficient}.Concerning CUSUM-UCB, we adhere to \cite{liu2018change} by fixing $\epsilon = 0.1$, setting the change detection threshold $b_{\textrm{CUSUM-UCB}} = \log\left(T/M-1\right)$, and establishing the uniform exploration rate $\gamma_{\textrm{CUSUM-UCB}} = \sqrt{MK \log{T}/T}$ as initially stated in \cite{besson2022efficient}. Additionally, in CUSUM-UCB, the change point detection involves averaging the first $H$ samples, where $H$ is set to 100. For GLR-UCB \cite{besson2022efficient}, we set $\gamma_{m, \textrm{GLR-UCB}} = \sqrt{mK \log{T}/T}$, where $m$ is the number of alarms. We utilize the threshold function $\beta(n, \delta) = \log{\left(n^{3/2}/\delta\right)}$ and set $\delta = 1/\sqrt{T}$. In our setup, for both the diminishing versions of M-UCB and CUSUM-UCB, we follow the parameter selection approach described earlier, except for the choice of the exploration rate. In this context, we opt for $\alpha=1$. 

%% file: appendix/related-work.tex
\section{Additional Related Work}
\textbf{{Structured Non-Stationary Bandits}}.
Another related line of research works is non-stationary bandit with structured reward changes, which are typically motivated by the dynamic behavior of real-world applications. 
For example, \citet{heidari2016tight} proposes the \textit{Rising bandit} problem, where the rewards are assumed to be a non-decreasing and concave function of the current time index and the number of pulls.
Subsequently, this model is extended to the stochastic setting by \citep{metelli2022stochastic}.
Another related setting is the \textit{Rotting bandit} \citep{levine2017rotting}, where the expected reward is a non-increasing function of the number of pulls.
Moreover, \citet{zhou2021regime} studies the regime switching bandit, where the rewards are jointly controlled by an underlying finite-state Markov chain.
However, the algorithms tailored to the above customized formulations are not directly applicable to the general piecewise-stationary MAB problem.

\textbf{{Bandit Quickest Change Detection}.}
Since the seminal works \citep{page1954continuous,lorden1971procedures}, the quickest change detection (QCD) problem, which involves identifying the change of distribution at an unknown time with minimal delay, has been a well studied detection problem of stochastic processes \citep{veeravalli2014quickest}. 
Bandit QCD, a variant of QCD problem recently proposed by \citep{gopalan2021bandit}, adds another layer of complexity to the conventional QCD by considering bandit feedback. 
A concurrent work \citep{xu2021optimum} also studies a similar setting, namely multi-stream QCD under sampling control, and proposes a myopic sampling policy that achieves a second-order asymptotically optimal detection delay.
Despite the above bandit QCD methods focusing mainly on achieving low detection delay rather than characterizing regret bounds, these recent progress could nicely complement the studies of piecewise stationary bandits.

%% file: appendix/simulation.tex
\section{Additional Simulations}\label{app:sim}


{\bf Regret Scaling in $K$.} We considered an environment with $T=20000$ and $M=5$ for various $K$, aiming to showcase dynamic regrets versus different $K$. In this experiment, expected rewards are generated randomly. Specifically, we randomly generated $5$ instances, averaging each instance over 50 times. Figure~\ref{fig:K} demonstrates that our method is not limited to working only in simple environments with small $K$ values but is adaptable to a broader range of scenarios.

{\bf Comepare to AdSwitch, ArmSwitch and Meta Algorithm.} These algorithms are indeed computationally quite complex, evidenced by the time complexity of $O(KT^4)$ in \cite{auer2019adaptively}, that of $O(K^2T^2)$ in \cite{abbasi2023new}, and the fact that recursive calls of the base algorithm are needed by the algorithm in \cite{suk2023tracking}. Besides, while being able to achieve near-optimal regret bound asymptotically, these elimination-based algorithms generally would not perform well when the time horizon $T$ is small, shown as Figure~\ref{fig:AxSwitch} and~\ref{fig:Meta}. Figure~\ref{fig:skip} compares our extension, which incorporates a skipping mechanism, with ArmSwitch, where the latter focuses on tracking the most significant arm switches. For our setup, we defined $\mu^{(i)}_{1}=0.8, 0.2$ for $i$ where $(i+1)\bmod 4=\left\{2,3\right\}, \left\{0,1\right\}$, and $\mu^{(i)}_{2}=0.4, 0.6$ for $i$ where $(i+2)\bmod 2=0,1$, as well as $\mu^{(i)}_{2}=0.4, 0.6$ for $i$ where $(i+3)\bmod 2=0,1$. In our parameter settings, we set $N_{I}=50$ and $\alpha=1$. The results clearly show that our performance significantly exceeds that of the ArmSwitch. The above discussion precisely constitutes the main reason we did not include these algorithms in our experimental comparison, as such algorithms with $T=20000\sim 100000$ would take too long to finish while showing results for small $T$ may look unfair for those excellent algorithms.

\begin{figure}
\centering
\begin{subfigure}{0.85\textwidth}
    \includegraphics[width=\textwidth]{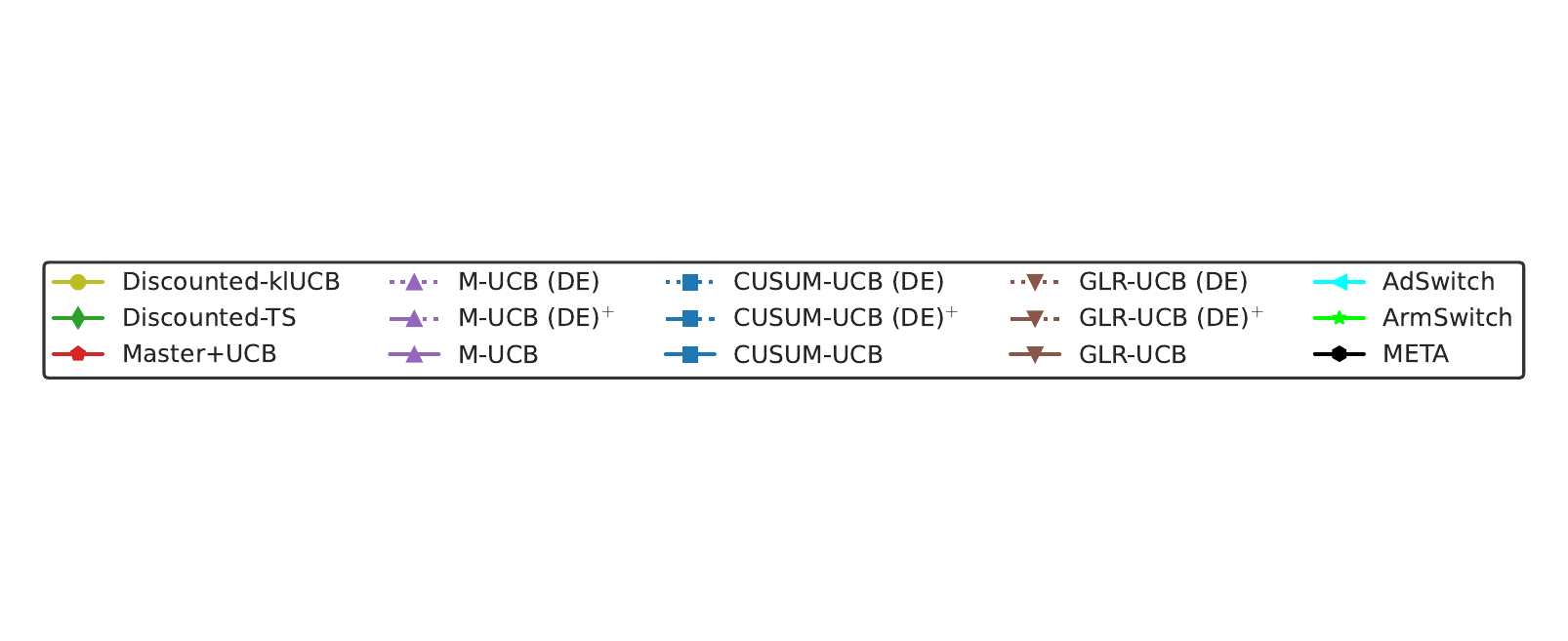}
    \vspace{-2cm}

\end{subfigure}
\begin{subfigure}{0.25\textwidth}
    \includegraphics[width=\textwidth]{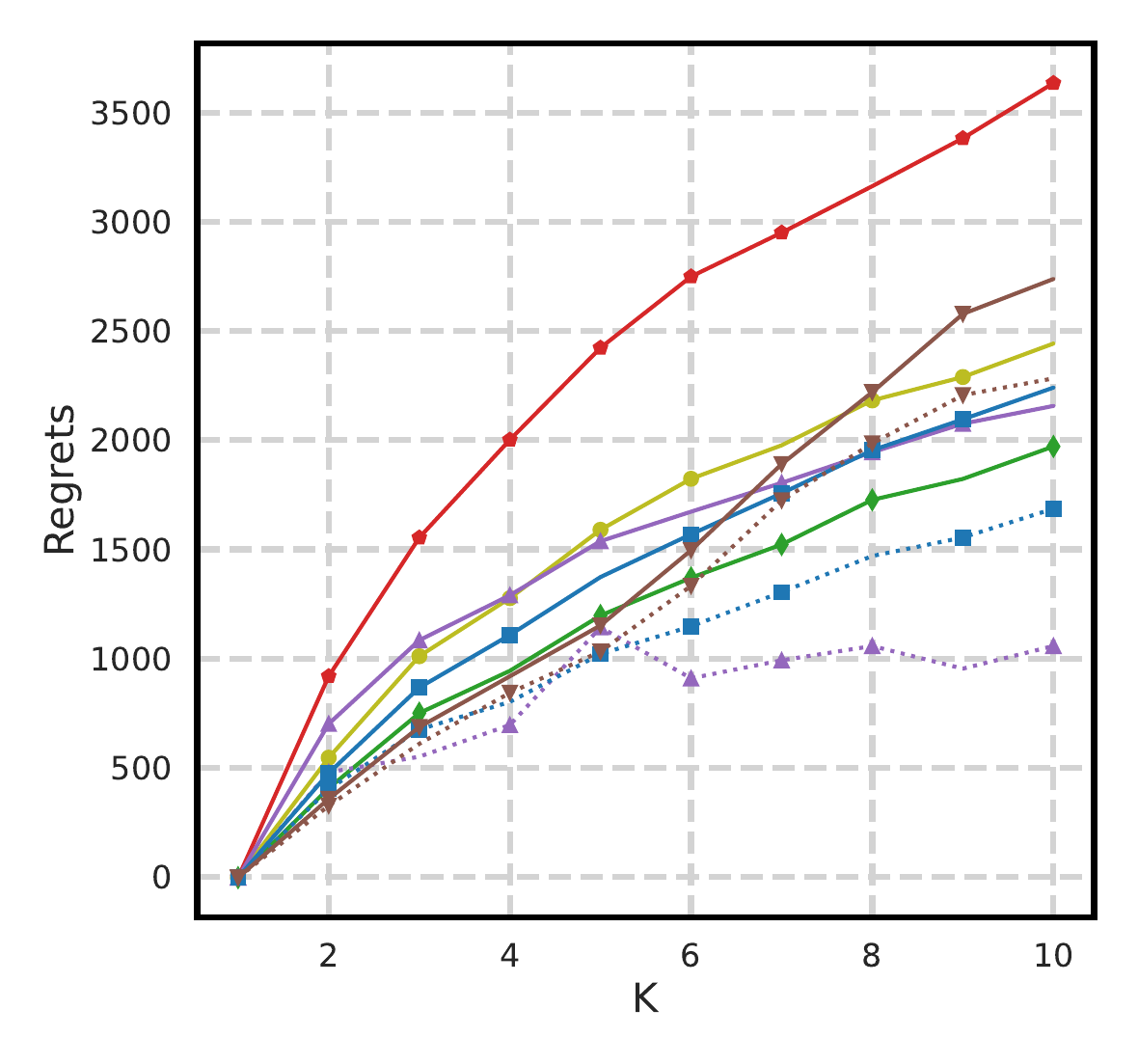}
    \vspace{-15pt}
    \caption{Scaling in $K$.}
    \label{fig:K}
\end{subfigure}
\hspace{-20pt}
\hfill
\begin{subfigure}{0.25\textwidth}
    \includegraphics[width=\textwidth]{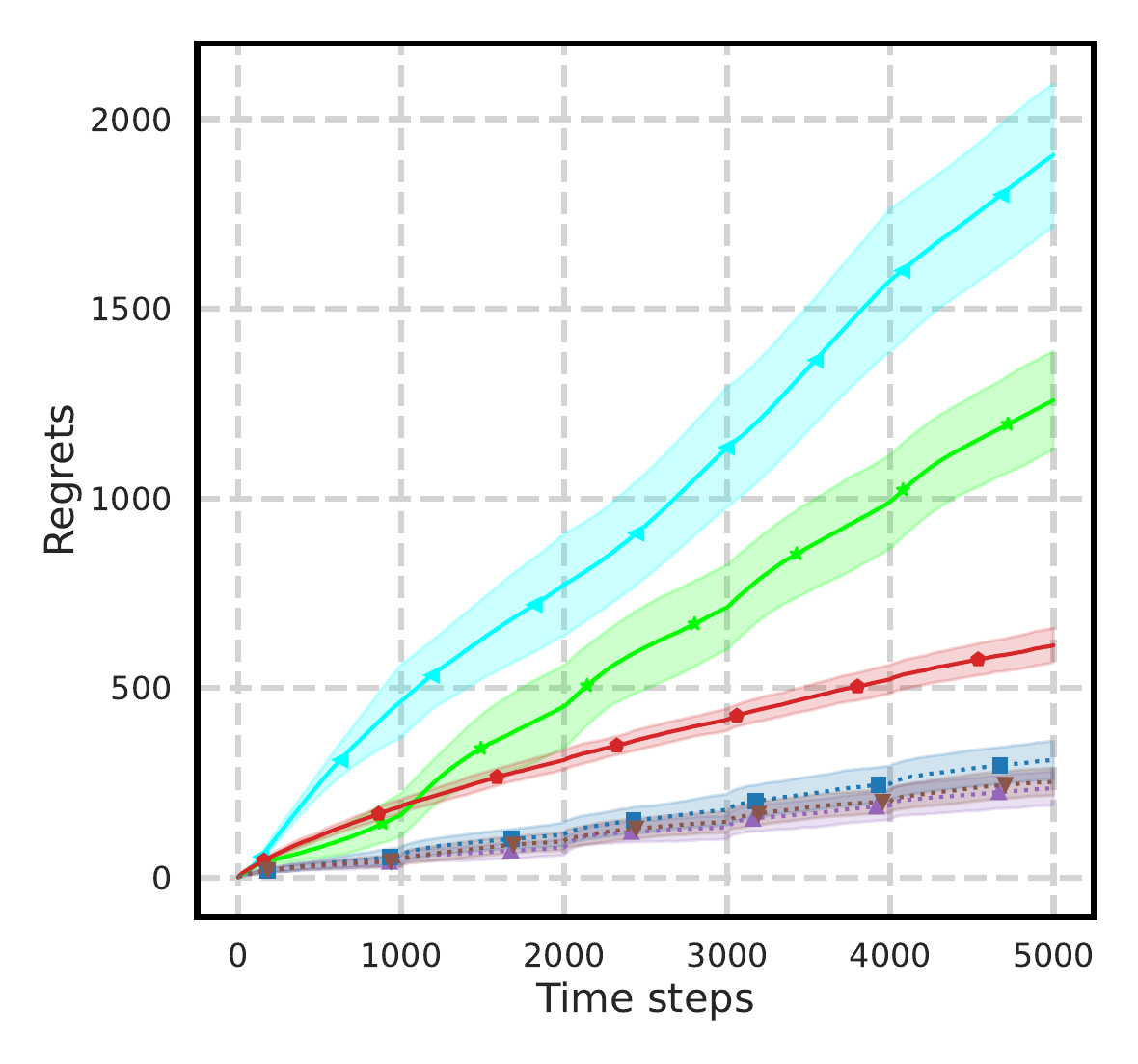}
    \vspace{-15pt}
    \caption{Compare to AdSwitch and ArmSwitch.}
    \vspace{-10pt}
    \label{fig:AxSwitch}
\end{subfigure}
\hspace{-20pt}
\hfill
\begin{subfigure}{0.25\textwidth}
    \includegraphics[width=\textwidth]{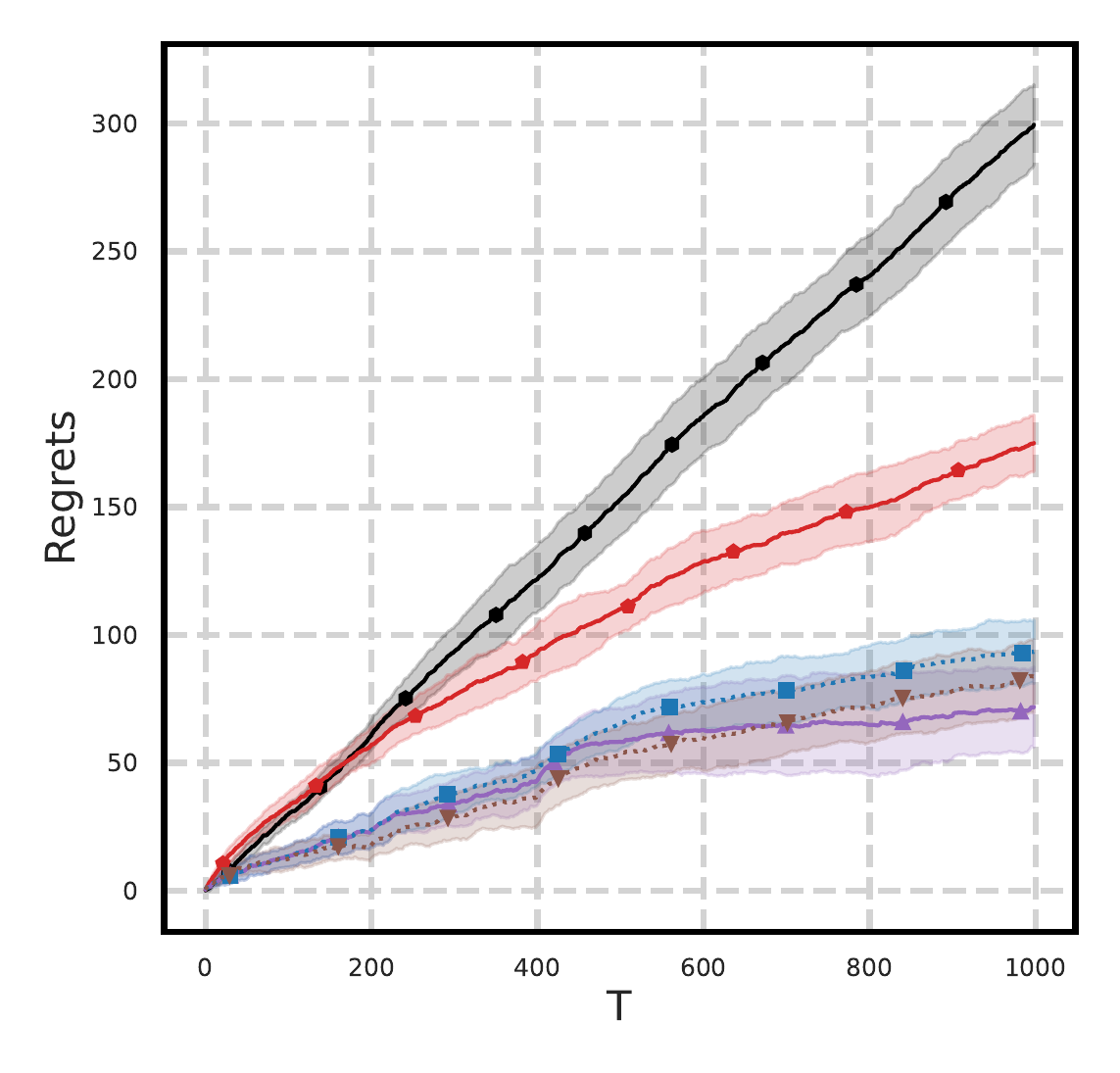}
    \vspace{-15pt}
    \caption{Compare to Meta.}
    \label{fig:Meta}
\end{subfigure}
\hspace{-20pt}
\hfill
\begin{subfigure}{0.25\textwidth}
    \includegraphics[width=\textwidth]{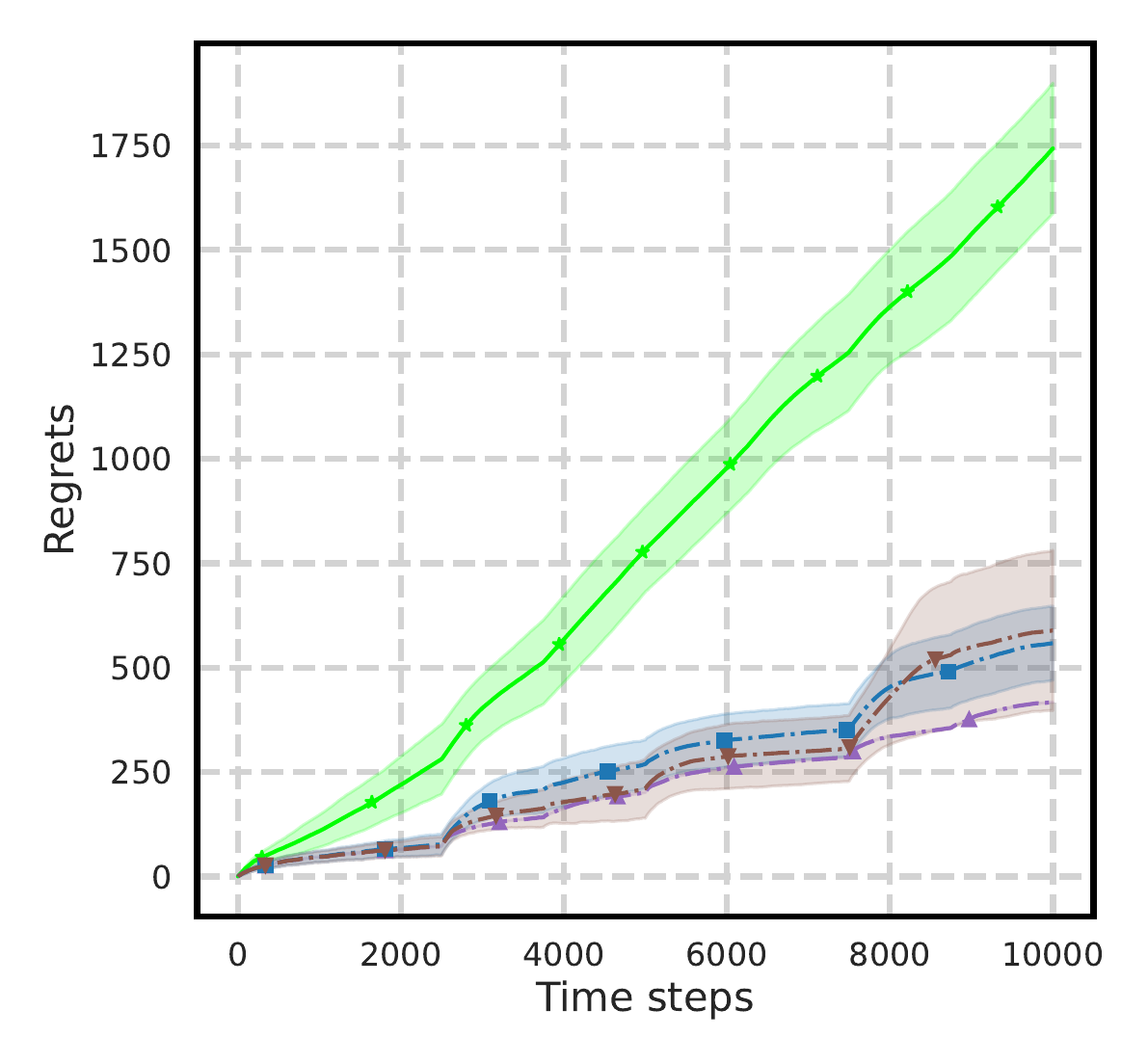}
    \vspace{-15pt}
    \caption{Add skipping mechanism.}
    \vspace{-10pt}
    \label{fig:skip}
\end{subfigure}
\hspace{-20pt}
\caption{Regret in synthetic environment and yahoo data set.}
\end{figure}
